%% file: main.tex
\documentclass[]{infiniai}
\usepackage{microtype}
\usepackage{amsfonts}
\usepackage{graphicx}
\usepackage{booktabs}
\usepackage{wrapfig}
\usepackage{float}
\usepackage{placeins}
\usepackage{minted}
\usepackage{colortbl}
\newcolumntype{P}[1]{>{\centering\arraybackslash}p{#1}} 
\newcommand{\gcell}{\cellcolor{gray!45}} 
\usepackage{amsmath}
\usepackage{amsthm}
\usepackage{amssymb} 
\usepackage{subcaption}

\usepackage{algpseudocode}
\usepackage[linesnumbered,lined,boxed,commentsnumbered,ruled,longend]{algorithm2e}
\newcommand{\sys}{\textsc{FlashRT}\xspace}

\DeclareMathOperator*{\argmin}{argmin} 
\usepackage[capitalize,noabbrev]{cleveref}
\crefname{appendix}{appendix}{appendices}
\Crefname{appendix}{Appendix}{Appendices}
\theoremstyle{plain}
\newtheorem{theorem}{Theorem}[section]
\newtheorem{proposition}[theorem]{Proposition}
\newtheorem{lemma}[theorem]{Lemma}

\theoremstyle{definition}

\theoremstyle{remark}

\crefname{theorem}{Lemma}{Lemmas}
\Crefname{theorem}{Lemma}{Lemmas}
\usepackage{enumitem}

\title{FlashRT: Agent Harness for Guiding Agents to Deploy Real-Time Multimodal Applications}

\author{Krish Agarwal$^1$, Zhuoming Chen$^1$, Yanyuan Qin$^2$, Zhenyu Gu$^2$, Atri Rudra$^3$, Beidi Chen$^1$}
\contact{\{krisha2, zhuominc, beidic\}@andrew.cmu.edu, \{yanyuan.qin, zhenyu.gu\}@amd.com, atri@buffalo.edu}
\affiliation{$^1$Carnegie Mellon University,
$^2$AMD,
$^3$University at Buffalo}

\abstract{
Real-time multimodal applications, including voice agents and interactive video generation, compose heterogeneous models into pipelines whose efficient deployment requires application-specific decisions about placement, streaming, and intra-model parallelism. Existing serving systems and auto-parallelism compilers commit to limited transformations and fixed workload assumptions, so achieving high performance on a new application requires hand-crafting an efficient implementation. We present \sys, an agent harness that guides coding agents to lift simple developer-written reference implementations into optimized multi-GPU deployments that flexibly weigh target metrics like latency and throughput. Using a new \textbf{chain-of-program} paradigm, \sys directs a generic coding agent through a multi-pass transformation process where an agent transforms the reference into an intermediate representation (IR) to capture data dependencies and persistent-state scopes, validates this IR via a sequential interpreter, and performs static analyses to identify candidate transformations. Then, the agent iteratively implements, verifies, and benchmarks each candidate under a measurement-gated optimization loop to produce effective deployments that span different hardware budgets. Across various applications, including video world models and multimodal LLMs, \sys converts reference implementations into highly efficient deployments, delivering up to \textbf{$\sim$70$\times$} latency reduction and \textbf{2.8$\times$} throughput improvement on NVIDIA B200 GPUs. On AMD MI355X GPUs, \sys matches the peak latency reduction while increasing peak throughput improvement to \textbf{3.6$\times$}, demonstrating that agent-driven optimization can be more scalable on platforms with less mature expert optimization. In fact, for Qwen3-Omni text-to-audio inference, \sys reduces response latency by \textbf{65\%} compared to the expert vLLM-Omni implementation on AMD MI355X.
}
\metadata[Github]{\url{https://github.com/Infini-AI-Lab/FlashRT}}
\metadata[Website]{\url{https://infini-ai-lab.github.io/flashrt-blog}}

\begin{document}

\maketitle

\input{sections/introduction}
\input{sections/related_work}
\input{sections/problem}
\input{sections/method}
\input{sections/experiments}
\input{sections/conclusion}

\clearpage

\bibliographystyle{assets/plainnat}
\bibliography{references}

\clearpage

\beginappendix
\crefalias{section}{appendix}
\crefalias{subsection}{appendix}
\crefalias{subsubsection}{appendix}
\input{sections/supp}

\end{document}

%% file: sections/introduction.tex
\section{Introduction}

As strong generative models emerge across various modalities, there is increasing adoption of real-time multimodal applications~\citep{robbyantteam2026lingbotworld, ball2025genie3, defossez2024moshi, kodaira2025streamdiffusion}, naturally followed by demand to serve them efficiently. Compared to extensive prior work on LLM serving~\citep{kwon2023vllm, yu2022orca, zheng2024sglang, zhong2024distserve}, multimodal applications present a fundamentally different design space: models of various modalities and architectures are composed into highly complex pipelines, each with its own deployment considerations, e.g., an application that prioritizes throughput (e.g., frame rate) may prefer full disaggregation while a latency-sensitive application may prefer finer-grained intra-component scheduling and co-location. 
These diverse applications cannot be served effectively under a single system that hosts only a limited set of deployment policies. Instead, new system efforts are required for each application, typically through manual implementation, but \textbf{steady growth in the number of new and diverse applications makes hand-crafting efficient deployments unscalable}.

Existing systems try to automate deployment to reduce human effort but fall short for three key reasons. \textbf{(1) Limited deployment strategies}: existing multi-modal frameworks like vLLM-Omni~\citep{yin2026vllmomni} and Cornserve~\citep{chung2026cornserve} provide useful abstractions for multimodal serving, but they restrict to static deployment policies (e.g., full disaggregation around high-level stage boundaries and full co-location within a stage), limiting efficacy on general latency or throughput targets. \textbf{(2) Limited workload coverage:} auto-parallelism frameworks like FlexFlow~\citep{jia2018datamodelparallelismdeep}, GSPMD~\citep{xu2021gspmd}, Alpa~\citep{zheng2022alpa}, and Unity~\citep{unger2022unity} can automatically choose sharding, placement, and pipelining, but they specialize to one fixed workload and cannot generalize to diverse multimodal applications. \textbf{(3) Optimization granularity:} existing compilation frameworks, like TVM~\citep{chen2018tvm} and TASO~\citep{jia2019taso}, though highly effective in optimizing deep learning workloads, are out of scope as they target only operator-level improvements. Moreover, such methods rely on a well-defined optimization granularity level, which does not exist for arbitrary multimodal applications.

\begin{wrapfigure}{r}{0.5\linewidth}
    \vspace{-6mm}
    \centering
    \includegraphics[width=\linewidth]{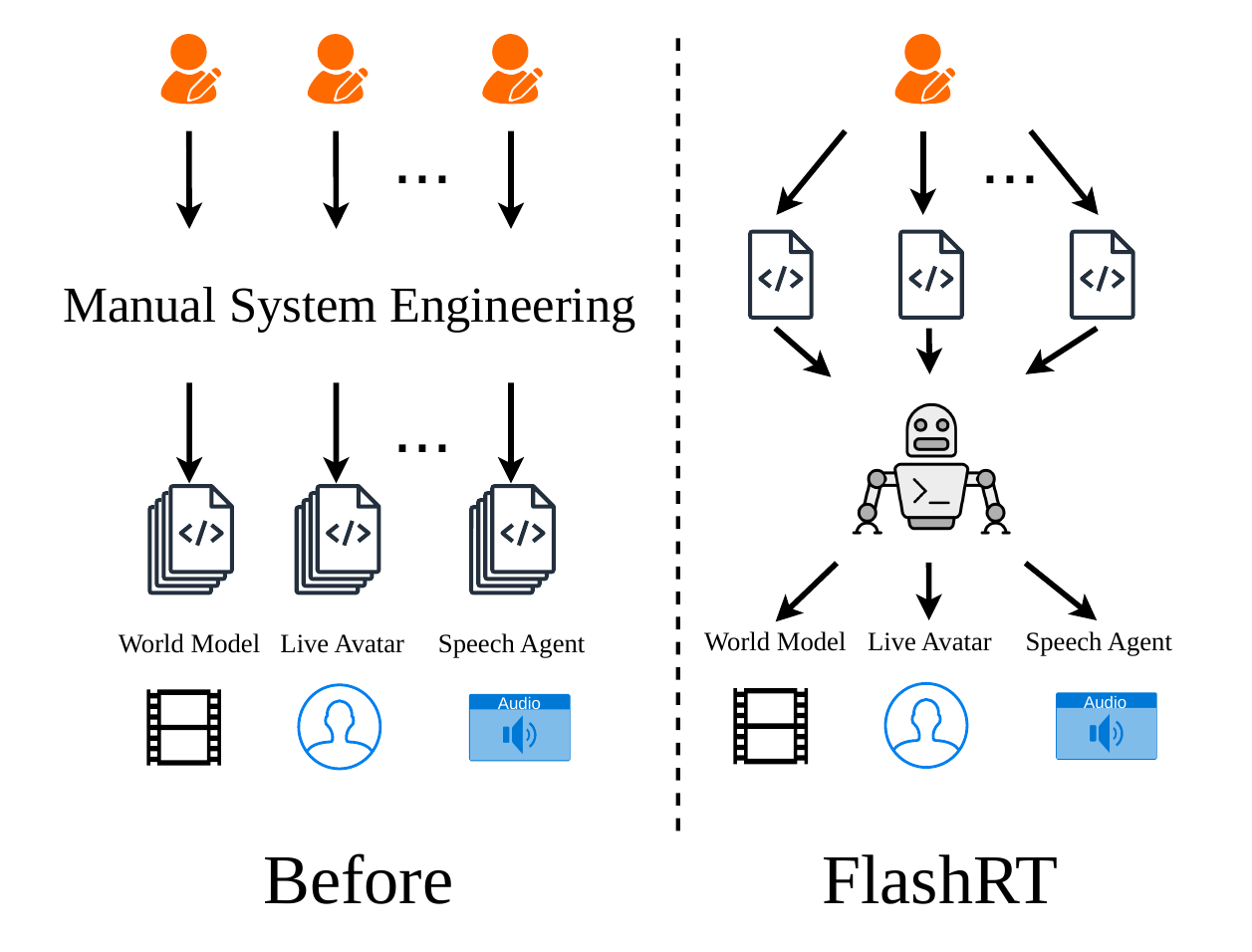}
    \vspace{-6mm}
    \caption{Previously, efficiently serving diverse multimodal applications required manual systems effort per application. \sys solves this by guiding a coding agent to design efficient deployments.}
    \label{fig:overview}
    \vspace{-3mm}
\end{wrapfigure}

An ideal serving system for multimodal applications should allow users to flexibly write intuitive, unoptimized, single-GPU implementations and automatically derive specialized system infrastructure for each application's individual needs. Such a system would directly address the key limitations of previous works by providing unconstrained deployment scope, workload flexibility, and adaptive optimization granularity. However, this is challenging to accomplish with rule-based systems since the deployment problem is NP-hard, as we elaborate in~\Cref{sec:problem}. This is especially problematic if deployment is performed at the finest granularity, e.g., kernel-level, which would result in too large a search space. However, any higher level of granularity is not even properly defined for a rule-based system, as the optimal granularity may be heterogeneous between different pipeline components.

Fortunately, AI agents~\citep{yang2024sweagent, wang2025openhands} have shown significant progress in automatically synthesizing code from simple specifications and performing low-level optimizations~\citep{ouyang2025kernelbench, liao2026kernelevolve, wei2025astra, shypula2024learningperformanceimprovingcodeedits}. Our key insight is that \textbf{coding agents are natively capable of writing efficient end-to-end implementations for complex multimodal pipelines}, as illustrated in~\Cref{fig:overview}. In particular, \textbf{agents can employ higher-level reasoning to organize a target pipeline at adaptive, heterogeneous granularities, reducing the deployment search space while enabling discovery of highly efficient implementations}. However, naively prompting an agent to directly transform a baseline implementation into an efficient deployment is not effective.
\begin{itemize}[itemsep=0.0pt,topsep=0pt,leftmargin=*]
    \item \textbf{Chain-of-program paradigm}: we identify that current agents cannot effectively perform direct translation in a single step. Inspired by existing approaches like chain-of-thought prompting for LLMs~\citep{wei2023cot} as well as compiler frameworks that progressively lower frontend code into a backend, we observe that agents are much more effective when required to first convert the user reference to a structured intermediate representation (IR), perform static analysis, and then translate the IR into effective deployments.
    \item \textbf{Application-grounded validation loop}: not only can agents fail to properly explore the diverse search space, but they may also easily produce incorrect solutions. We find that agents deliver stronger guarantees by enforcing a self-driven loop where the agent explores, implements, and measures various hypotheses; importantly, the agent can design an application-specific test harness that drives simulated user input through the backend implementation, which is critical for grounding validation and benchmarking results in the application-specific user experience. 
\end{itemize}
Combining these insights, we introduce \sys, a system for serving multimodal applications that guides a generic coding agent through a structured workflow to automatically convert highly flexible user implementations into efficient deployments. Across a range of multimodal applications, we show that \sys synthesizes low-latency deployments that are competitive with expert-designed systems. Compared to third-party serving systems and other expert-written systems, \sys generates deployments that can flexibly tradeoff between time-to-first-output (TTFO) latency and system throughput. On the applications we evaluate, \sys successfully converts baseline implementations into highly efficient deployments, reaching \textbf{$\sim$70$\times$} latency reduction and \textbf{3.6$\times$} throughput gain. Importantly, this agentic system alone can generalize to highly diverse applications on both NVIDIA B200 and AMD MI355X hardware, while still requiring no human intervention beyond supplying a reference implementation. As such, \sys proves to be an effective framework that can substantially reduce manual systems effort.

%% file: sections/related_work.tex
\section{Related Work}
\label{sec:related}

\paragraph{Real-time multimodal serving.}
Many new applications compose heterogeneous models into real-time pipelines. Examples include speech-to-speech voice agents~\citep{defossez2024moshi, xu2025qwenomni, communication2023seamless}, audio-driven avatar generation~\citep{zhang2025musetalk, tian2024emo, xu2024hallo, peng2024synctalk, huang2026liveavatar}, real-time video world modeling~\citep{kodaira2025streamdiffusion, ball2025genie3, sun2026worldplay, robbyantteam2026lingbotworld}, and embodied vision-language-action policies~\citep{black2026pi0}. Multimodal serving systems, including vLLM-Omni~\citep{yin2026vllmomni}, Cornserve~\citep{chung2026cornserve}, and ModServe~\citep{qiu2025modserve}, can express each application as a graph over a small fixed set of well-defined stages. While this provides clean abstractions for new applications, it leads to restrictive deployment policies that are not always optimal, such as required inter-stage disaggregation.

\paragraph{Auto-parallelism and {ML} compilers.}
Auto-parallelism systems such as FlexFlow~\citep{jia2018datamodelparallelismdeep}, GSPMD~\citep{xu2021gspmd}, Alpa~\citep{zheng2022alpa}, and Unity~\citep{unger2022unity} jointly choose data, model, and pipeline parallelism for distributed deep learning, but their search spaces and cost models are calibrated to dense DNN training and do not transfer to inference pipelines composed of heterogeneous models with disjoint runtimes. Pipeline-parallel methods like PipeDream~\citep{narayanan2019pipedream} provide useful primitives but assume a single homogeneous training graph. Operator-level compilers such as TVM~\citep{chen2018tvm}, TASO~\citep{jia2019taso}, and Halide~\citep{ragankelley2013halide} aggressively rewrite kernel-level computation, but their granularity for optimization does not touch important considerations like device placement or intra-model stage overlap.

\paragraph{Coding agents for performance optimization.}
Agent systems for software engineering~\citep{yang2024sweagent, wang2025openhands, hong2024metagpt, shinn2023reflexion} have demonstrated that LLMs paired with environment feedback can navigate large codebases and iteratively repair their own outputs. For GPU kernel generation and optimization~\citep{ouyang2025kernelbench, liao2026kernelevolve, wei2025astra, lange2025robustagenticcudakernel, baronio2025kevin, li2026cudal1, li2025autotriton, yang2025cudaforge}, recent work has explored using agents with traditional compiler pass selection~\citep{lin2025awarecompiler}, scientific-discovery search~\citep{novikov2025alphaevolve, Romera-Paredes2024Mathematical}, and reward design~\citep{ma2024eureka}. \sys differs from these in that it optimizes end-to-end multimodal deployment across multiple model runtimes.

%% file: sections/problem.tex
\section{Problem Formulation}
\label{sec:problem}

\sys takes as input a synchronous single-GPU reference implementation
$P_{\mathrm{ref}}$ of a multimodal application. $P_{\mathrm{ref}}$ defines the
application semantics and any persistent state scopes (e.g., caches or streaming buffers). \sys returns a multi-GPU deployment
$P_{\mathrm{dep}}$ that may use disaggregation, co-location, streaming, and intra-model parallelism while preserving the behavior of
$P_{\mathrm{ref}}$.

We model an application as a task graph $G=(\mathcal V,E)$, where each
$v\in\mathcal V$ is a computation region. The application processes input in
batches indexed by $j$, and the graph induces one operation instance $(j,v)$
per batch. Each edge
$(u,v,\lambda)\in E\subseteq\mathcal V\times\mathcal V\times\{0,1\}$ denotes a
data or state dependence, with $\lambda=0$ an intra-batch (data) dependence of
$(j,v)$ on $(j,u)$ and $\lambda=1$ a cross-batch (state) dependence of $(j,v)$
on the previous instance $(j-1,u)$ (e.g., a KV cache that batch $j$ inherits
from batch $j-1$). A deployment specifies a placement
\[
\rho:\mathcal V\to\mathcal R,
\]
where each resource $r\in\mathcal R$ may be a single- or multi-GPU group,
and a non-preemptive schedule
\[
A_r:\mathbb R_{\ge 0}\to
\{(j,v):\rho(v)=r\}\cup\{\varnothing\},
\]
where $\varnothing$ indicates the resource is idle. Let $S_{j,v}$ be the start time of instance $(j,v)$. Each edge $(u,v,\lambda)\in E$ imposes
\[
S_{j,v}
\;\ge\;
S_{j-\lambda,u}+\tau_u(\rho)+\kappa_{u,v}(\rho),
\]
whenever the predecessor instance $(j-\lambda,u)$ exists. Here
$\tau_u(\rho)$ is the placement-dependent execution time of $u$, and
$\kappa_{u,v}(\rho)$ captures synchronization or transfer cost. The schedule
must also satisfy resource capacity, i.e., each resource non-preemptively executes at most one assigned instance at a time.

The objective is to design a deployment that minimizes application-level serving
metrics, such as response latency or throughput, subject to correctness constraints. More detailed formulations are provided in~\Cref{sec:analysis:setup}, which formalizes these two metrics and shows that latency is governed by the critical path from input to output, while sustainable throughput is governed by the most heavily loaded resource. \Cref{sec:analysis:app1,sec:analysis:app2} apply these bounds to representative streaming pipelines, and both reveal the same latency-throughput tradeoff: co-locating operations on a shared resource shortens the critical path, whereas disaggregating them onto separate resources and pipelining across batches reduces the load on the most heavily loaded resource, so no single deployment is best for both targets. This deployment problem is hard even in restricted cases. With a single batch,
no dependencies, no communication costs, and a makespan objective (minimizing the time to finish all operations), it reduces to classical non-preemptive multiprocessor scheduling, which is
NP-hard~\citep{lenstra1977complexity}. Multimodal serving strictly generalizes this setting: operations have data and state dependencies, execution times depend on placement and parallelization, and communication costs depend on co-location. A proof is provided in~\Cref{sec:analysis:np_hard}. Because execution and communication costs couple the deployment decisions and finding an optimum is NP-hard, an efficient deployment can neither be prescribed by a fixed policy nor solved exactly; it must instead be discovered per application by jointly reasoning about program structure, resource assignment, scheduling, and parallelism.

%% file: sections/method.tex
\section{\sys}
\label{sec:method}

\begin{figure}[!t]
    \centering
\vspace{2pt}
    \begin{subfigure}[t]{0.58\textwidth}
        \vspace{0pt}
        \centering
\begin{minted}[
    fontsize=\scriptsize,
    fontfamily=cmtt,
    breaklines,
    frame=single,
    bgcolor=gray!8,
    numbersep=4pt,
]{python}
def run_conversational_agent(user_prompt_audio):
    user_prompt_text = run_asr(user_prompt_audio, HF)
    response_text = run_llm(user_prompt_text, vLLM)
    response_audio = concat([
        audio_chunk for audio_chunk in
        run_tts(response_text, vLLM_Omni)
    ])
    response_video = _run_liveavatar_s2v(response_audio)
    return combined(response_audio, response_video)

def _run_liveavatar_s2v(audio):
    video = []
    for idx in range(0, len(audio_chunk), MODEL_BATCH_SIZE):
        x = audio[idx : idx + MODEL_BATCH_SIZE]
        for step in range(MODEL_DIFFUSION_STEPS):
            x = denoise(x, step=step)
        x = run_vae(x)
        video.append(x)
    return concat(video)
\end{minted}
    \caption{Example user base implementation inputted to \sys.}
        \label{fig:ex_user_base_impl}
    \end{subfigure}%
    \hfill
    \begin{subfigure}[t]{0.38\textwidth}
        \vspace{0pt}
        \centering
        \includegraphics[width=0.8\linewidth]{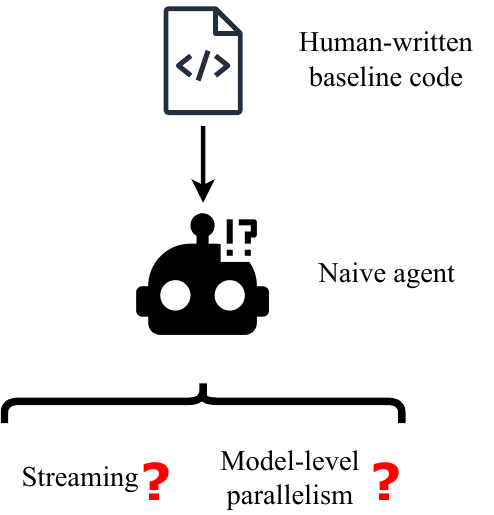}
        \caption{The naive agent finds only a subset of the available optimization axes, and it will not consider composing strategies.}
        \label{fig:naive_agent}
    \end{subfigure}

    \vspace{1em}

    \begin{subfigure}[t]{\textwidth}
        \vspace{0pt}
        \centering
        \includegraphics[width=\linewidth]{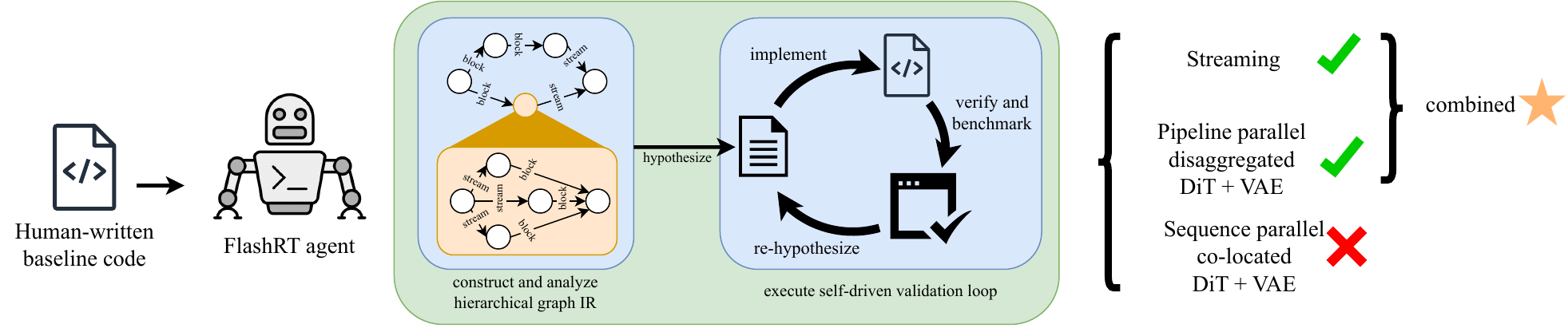}
        \vspace{2pt}
        \caption{\sys transforms the baseline implementation into an IR to perform analysis. Then, it enters a self-driven validation loop to test various hypotheses. This allows the agent to consistently find diverse optimization strategies and consider compositions of strategies that improve the specified target (throughput in this case).}
        \label{fig:flashrt_agent}
    \end{subfigure}

    \caption{Overall view of different agent workflows. (a) shows an example input to the ideal system. (b) shows the result of naively tasking an agent to optimize a base implementation. (c) shows how failure cases are addressed by \sys's agent workflow.}
    \label{fig:agent_workflow}
\end{figure}

Converting a developer-written sequential reference for a multimodal pipeline into an efficient multi-GPU deployment requires reasoning about each application's specific structure, which is a task naturally suited to a coding agent. Unlike rule-based serving systems and auto-parallelism compilers, a coding agent can read an arbitrary reference, reason about its structure, and produce a deployment specialized to the application at hand. However, naively handing the agent a reference and directly tasking it to produce an optimized version can fail. In~\Cref{sec:method:case_study}, we present a case study analysis for several failure modes when performing the naive agent deployment. Then in~\Cref{sec:method:ir,sec:method:iteration}, we describe how the \sys framework addresses these failures, namely through hierarchical planning via an IR and a self-driven validation loop, respectively.

\subsection{Case study}
\label{sec:method:case_study}

We consider a face-to-face conversational agent that generates a talking-head video response to a user's spoken query. The application pipeline consists of the following stages: \textbf{1.} The user’s speech input is processed by an automatic speech recognition (ASR) module to transcribe the audio into text. \textbf{2.} The transcribed text prompt is fed into a large language model (LLM) to generate a textual response. \textbf{3.} The generated text response is passed through a text-to-speech (TTS) model to synthesize response audio. \textbf{4.} The synthesized audio is provided to a sound-to-video (S2V) generative model, which produces a video of a person speaking the response. \textbf{5.} The generated audio and video streams are synchronized and returned to the user as the final response.

\Cref{fig:ex_user_base_impl} shows an example of what a user base implementation (the input to the system), would look like. Concretely, we use Qwen3-ASR~\citep{shi2026qwen3asr} for ASR, Qwen3-4B~\citep{yang2025qwen3} as the LLM, Qwen3-TTS~\citep{hu2026qwen3tts} for TTS, and LiveAvatar~\citep{huang2026liveavatar} for S2V.

We consider the case where the response latency budget is flexible, and the agent must instead optimize for throughput (i.e., frame rate). This objective will naturally prefer deployments that use inter-module streaming and pipeline parallelism. We manually identify two axes for optimization:
\begin{itemize}[itemsep=0.0pt,topsep=0pt,leftmargin=*]
    \item \textit{Application level}: once the LLM response is received, the TTS model can iteratively produce chunks of audio output rather than producing the whole audio at once; the LiveAvatar S2V model is autoregressive and also processes input audio in chunks, meaning that the TTS model can be streamed into the LiveAvatar model. Also, generated chunks can be displayed while future chunks are still being processed, meaning the synced audio + video output can be streamed into a display buffer.
    \item \textit{Model level}: the LiveAvatar model is a video generation model built on a diffusion transformer (DiT), and it consists of DiT and variational autoencoder (VAE) components; since the VAE does not share any internal state (e.g., KV cache) with the DiT, the DiT can be run on an input batch in parallel with the VAE on the previous batch on separate GPUs. Additionally, LiveAvatar is trained such that each diffusion step uses an independent KV cache, so the different steps can be deployed with pipeline parallelism (PP), i.e., one GPU per diffusion step. We formally analyze the latency--throughput implications of these two model-level choices in~\Cref{sec:analysis:app1} (co-locating versus disaggregating the DiT and VAE) and~\Cref{sec:analysis:app2}, which compares sequence parallelism (SP) with PP across diffusion steps.
\end{itemize}
When we provide the agent with a synchronous, single GPU baseline implementation of this pipeline and naively instruct it to find a more efficient deployment, we observe two key failure modes (also highlighted in~\Cref{fig:naive_agent}):
\begin{itemize}[itemsep=0.0pt,topsep=0pt,leftmargin=*]
    \item \textbf{The agent does not recognize any model-level parallelization opportunities, so it fails to address low throughput.} Identifying that the diffusion steps can be pipelined requires reasoning about which parts of the pipeline share persistent state, since the steps are separable precisely because each maintains its own KV cache, as is the VAE because it shares no state with the DiT. An agent that reads the reference as ordinary code does not necessarily carry out such analysis on its own. Ideally, the agent would recognize pipeline parallelism opportunities based on independence of KV caches between diffusion steps and no shared state with the VAE.
    \item \textbf{The agent does not explore diverse optimization patterns and terminates upon discovering simple optimizations.} The agent discovers that it can stream the S2V-generated video chunks into the output buffer, but it overlooks that TTS can also be streamed into S2V for the same reason; the agent also never attempts to combine several strategies into a single deployment. Ideally, the agent should discover multiple optimization axes in a single run, and it should also compose different axes to form the most effective deployment.
\end{itemize}
These failure modes are highly problematic because resolving them requires explicit human intervention to help the agent identify optimization opportunities it previously missed. To resolve this,~\Cref{sec:method:ir,sec:method:iteration} present how we design the \sys agent workflow so that these failure cases are eliminated while maintaining full agent autonomy. These design principles are also visualized in~\Cref{fig:flashrt_agent}. In~\Cref{sec:exp:ablation} we return to these two failure modes and show that removing either component of the \sys workflow reproduces them.

\subsection{Chain-of-Program Paradigm for Hierarchical Planning}

\label{sec:method:ir}

We observe that failure cases where the naive agent fails to recognize important optimization opportunities (e.g., S2V pipeline parallelism) stem from the fact that \textbf{the agent is expected to directly transform a baseline implementation to an effective deployment in one monolithic step}. This differs from how a human would approach the problem: humans employ high-level reasoning to understand end-to-end workflows at varying levels of granularity to design effective implementations. This is analogous to widely adopted chain-of-thought~\citep{wei2023cot} mechanisms that allow language models to reason about a prompt before providing a final response. It is also what motivates the use of intermediate representations in conventional compilers, as it is simpler (and more scalable) to design a compiler that progressively lowers a frontend implementation into a backend over multiple passes. We observe that, by instructing the agent to first convert the baseline implementation to a properly structured intermediate representation (IR) and then convert the IR to a set of candidate deployments, the agent is more robust at discovering all possible axes for efficient implementation. We denote this paradigm as \textbf{chain-of-program}.

\paragraph{IR design.}
The IR represents a pipeline as one or more hierarchical, directed acyclic graphs. Specifically, the IR represents an application workflow using nodes to encapsulate self-contained operations and edges to denote data dependencies. These graphs instantiate the task graph $G=(\mathcal V,E)$ of~\Cref{sec:problem}, with nodes as the computation regions $\mathcal V$ and dataflow edges as the intra-batch data dependencies; the cross-batch state dependencies are supplied by the node annotations below. The IR exposes three structural properties of the pipeline that together drive the subsequent agentic analysis.
\begin{itemize}[itemsep=0.0pt,topsep=0pt,leftmargin=*]
    \item \textbf{Graph hierarchy}: The agent must generate an IR graph in a hierarchical manner, e.g., a top-level graph modeling the overall application workflow and multiple nested graphs specifying intra-module computations. First, a hierarchical graph IR allows the baseline implementation to remain flexible, e.g., it can compose custom model implementations with opaque third-party serving systems. Second, by requiring the agent to model both high-level application graphs and nested intra-module graphs for exposed model runtimes, the agent is forced to consider various levels of granularity for optimization, preventing it from focusing only on high-level application workflow. This hierarchy also sets the granularity of the nodes $\mathcal V$: exposing finer nodes (e.g., a model's internal stages) is what lets a deployment place or parallelize them separately, which a flat, module-level graph would preclude.
    \item \textbf{Node-level state annotation:} For each IR node, the agent is required to specify the set of persistent state reads and writes corresponding to the node's operations (this excludes transient internal state). These reads and writes recover the cross-batch (state) dependencies of~\Cref{sec:problem}, i.e., the $\lambda=1$ edges: two nodes that touch the same persistent state (e.g., a KV cache) are joined by such an edge and must be ordered across batches, whereas nodes with disjoint state carry no such edge and can be disaggregated and executed in parallel. The agent is therefore deterred from missing key opportunities related to disaggregation and pipeline parallelism.
    \item \textbf{Edge-level streaming annotation:} For each data dependency edge between nodes, the agent is required to explicitly mark the edge as either ``blocking" or ``streaming". For example, an LLM $\to$ TTS edge from the example in~\Cref{sec:method:case_study} would be ``blocking", while a TTS $\to$ S2V edge would be ``streaming". Similar to the previous point, these explicit annotations deter the agent from missing critical design opportunities: streaming edges mark where a consumer can begin before its producer finishes, so the two can overlap on separate resources via cross-batch pipelining, whereas blocking edges must serialize and therefore favor co-location.
\end{itemize}
\paragraph{Agentic IR analysis.}
In addition to having the agent translate the user implementation to an IR, the agent is instructed to perform static analysis on the resulting IR. This analysis critically includes several static tools that we provide to the agent; specifically, these tools analyze the agent's generated IR graph and the data dependencies that originate from graph edges and/or shared persistent state to statically identify parallelizable nodes and streaming opportunities; this tool allows the agent's analysis workflow to be more reliable and deterministic in surfacing key deployment design principles that it can try to implement. As a side note, we also provide the agent with an IR interpreter that takes the agent's generated IR graph and executes it sequentially based on topological order; the agent can execute the baseline implementation directly and its generated IR graph with the interpreter on the same sample inputs to compare outputs and can confirm the correctness of the IR it generates.

\subsection{Application-grounded validation loop}
\label{sec:method:iteration}

We find that even with a well-defined candidate space surfaced by IR analysis, an unguided agent may still focus on only one axis of optimization, or it may settle for an inefficient implementation of a strategy it has identified. The first of these is the failure case shown in~\Cref{sec:method:case_study}, where the agent implements one application-level streaming optimization and terminates without considering other axes or their combinations. Many existing works on agentic kernel generation~\citep{ouyang2025kernelbench, liao2026kernelevolve, wei2025astra} and local code optimization~\citep{shypula2024learningperformanceimprovingcodeedits} introduce self-driven validation loops where the agent continuously iterates on a solution, verifies correctness, and benchmarks performance. Critically, while verifying correctness and benchmarking performance is well-defined in these use cases, these are highly application-specific for multi-modal pipelines. Therefore, we propose a self-driven validation loop where, not only does the agent propose hypotheses and implement them in every iteration, but the agent also designs a test harness that is grounded in the user application experience to evaluate key metrics.

\paragraph{Verification and benchmarking.}
In \sys's self-driven loop, each iteration requires the agent to verify correctness of and benchmark its implementation. Concretely, we observe that, although the agent is naturally unable to interact with an application backend via a user-facing frontend, the agent can in fact implement test cases that write simulated input to the backend's input buffers and read corresponding output from its output buffers. This is analogous to how a user-facing frontend would write user interactions to the backend's input buffers and read from its output buffers to display results to the user, effectively grounding the agent's experiments in the true user experience. Through these test cases, the agent can validate correctness by driving the same simulated inputs to the baseline backend and the agent's generated backend; furthermore, the agent can measure end-to-end latency and throughput by collecting timing statistics when writing to input buffers and reading from output buffers. By requiring the agent to perform this structured verification, the agent's self-driven validation loop is strengthened since its hypotheses are grounded in application-specific measurements that mimic real user interactions.

\paragraph{Iteration step.}
In \sys, the agent not only performs self-driven iteration on its solution, but it facilitates the loop using a self-evolving variant queue that expands the diversity of strategies explored and improves the agent's robustness. First, the variant queue is initialized from the IR analyses. Each iteration begins with a hypothesis, i.e., naming the transformation analysis that justifies its legality and the bottleneck it targets. The agent then implements this in an isolated environment and verifies element-wise output equivalence of its implementation against the baseline implementation over a set of sample inputs; if errors are discovered, the agent proceeds to debugging its implementation. After producing a correct implementation, the agent benchmarks its implementation on sample inputs and uses the result to update the queue, appending any new variants the outcome suggests and re-ranking pending ones by expected impact. The loop terminates only when every queued variant has been measured or removed for a documented reason. Through this workflow, the agent is required to fully validate/measure its hypotheses and consider both diversity of implementation design as well as potential for strategy compositions through the reevaluation mechanism.

%% file: sections/experiments.tex
\vspace{-5pt}
\section{Experiments}
\label{sec:exp}

We evaluate \sys on a diverse set of real-time multimodal applications and answer five questions: \textbf{(1)} Can the agent generate highly efficient implementations that deliver high throughput at low time-to-first-output (TTFO) latency? \textbf{(2)} Is the agent able to generalize to a variety of real-time multi-modal applications? \textbf{(3)} Does the agent framework generalize across hardware? \textbf{(4)} Which parts of the harness are responsible for its effectiveness? \textbf{(5)} How robust is the result to the agent configuration and to run-to-run variation? We organize the experiments as follows:
\begin{itemize}[itemsep=0.0pt,topsep=0pt,leftmargin=*]
    \item In~\Cref{sec:exp:liveavatar}, we present a case study, matching the application studied in~\Cref{sec:method:case_study}, of how the agent is able to compose several strategies to achieve \textbf{$\sim$70$\times$} latency reduction from the baseline and high theoretical frame rate.
    \item In~\Cref{sec:exp:generalization}, we evaluate the agent on four additional applications, each with different compute structures. We find that the agent continues to deliver principled deployments that can jointly reduce latency while boosting throughput.
    \item In~\Cref{sec:exp:hardware}, we execute the agent pipeline separately on AMD MI355X for the same set of applications and find that it recovers the same class of efficient deployments, and the same latency and frame-rate tradeoffs, that it finds on NVIDIA B200.
    \item In~\Cref{sec:exp:ablation}, we ablate the two components of the harness and find that each is load-bearing on the agent's ability to produce efficient deployments.
    \item In~\Cref{sec:exp:robustness}, we measure how the \sys harness behaves under a different agent configuration and across independent runs.
\end{itemize}

\subsection{Setup}
\label{sec:exp:setup}

\paragraph{Hardware.} Unless otherwise noted, experiments are run on a single node with 8 NVIDIA B200 GPUs;~\Cref{sec:exp:hardware} additionally evaluates the full pipeline on a node with 8 AMD MI355X GPUs. The number of GPUs allocated to the agent varies across experiments (between 1 and 8), reflecting the different scale at which we evaluate each pipeline. We report the GPU count for each result.

\paragraph{Agent configuration.} Unless otherwise noted, we use Anthropic Claude Code~\citep{anthropic2026claudecode} with Claude Opus 4.8~\citep{anthropic2026opus48news}, using adaptive thinking with \texttt{effort=max}. The agent runs in an isolated workspace with Claude Code's Auto permission mode enabled, which permits unattended tool use subject to Claude Code's automated action classifier. Each session starts from scratch with no memory of previous runs and no visibility into runs on other applications. The starting prompt only specifies the application to optimize and the GPU budget; it provides no hints about possible deployment strategies. Each application is supplied with a synchronous, single-GPU reference implementation that performs no inter-component streaming, and the agent is responsible for self-discovering all deployment patterns.~\Cref{sec:exp:robustness} additionally evaluates a second agent configuration, OpenAI Codex~\citep{openai2026codexcli} with GPT-5.6 Sol~\citep{openai2026gpt56sol}.

\subsection{Case Study: Face-to-Face Conversational Agent}
\label{sec:exp:liveavatar}

\begin{table}[t]
\centering
\caption{
Comparison between the sequential user baseline and several deployments proposed by the \sys agent for face-to-face conversational agent application. Frame rate is omitted for the sequential baseline because the full output video is generated offline and becomes
playable only after completion. \sys deployments can produce frames at a higher FPS than the videos the S2V model is trained on, so their frame rate is a theoretical sustainable rate rather than the actual playback rate.
}
\label{tab:liveavatar}

\setlength{\tabcolsep}{6pt}
\renewcommand{\arraystretch}{1.15}

\small
\begin{tabular}{c l c c}
\toprule
\textbf{\# GPUs} &
\textbf{Deployment} &
\textbf{Latency (s)} $\boldsymbol{\downarrow}$ &
\textbf{Frame rate (FPS)} $\boldsymbol{\uparrow}$ \\
\midrule

1 &
Baseline (sequential, no streaming)
& 117.10
& -- \\

2 &
\sys~(streaming)
& \textbf{1.69}
& 30.92 \\

6 &
\sys~(streaming + S2V PP)
& 1.72
& \textbf{133.48} \\

\bottomrule
\end{tabular}
\vspace{-5pt}
\end{table}

\begin{table}[t]
\centering
\caption{Placement $\rho$ for the LiveAvatar conversational-agent deployments. The sequential baseline (omitted) places all nodes on one GPU. At 6 GPUs, the S2V DiT's four denoising steps (DiT$_1$--DiT$_4$) are pipelined one per GPU (the same per-step decomposition of~\Cref{sec:analysis:app2}), while the ASR shares a GPU with the VAE and the LLM shares a GPU with the TTS.}
\label{tab:place_liveavatar}
\small
\setlength{\tabcolsep}{4pt}
\renewcommand{\arraystretch}{1.1}
\begin{tabular}{@{}lccccc@{\hspace{0.6em}}!{\vrule}@{\hspace{0.6em}}cccccccc@{}}
\toprule
& \multicolumn{5}{c}{\textbf{2 GPU (streaming)}} & \multicolumn{8}{c}{\textbf{6 GPU (streaming + S2V PP)}} \\
\cmidrule(lr){2-6} \cmidrule(lr){7-14}
GPU & ASR & LLM & TTS & DiT & VAE & ASR & LLM & TTS & DiT$_1$ & DiT$_2$ & DiT$_3$ & DiT$_4$ & VAE \\
\midrule
G0 & X &  &  & X & X &  &  &  & X &  &  &  &  \\
G1 &  & X & X &  &  &  &  &  &  & X &  &  &  \\
G2 & \gcell & \gcell & \gcell & \gcell & \gcell &  &  &  &  &  & X &  &  \\
G3 & \gcell & \gcell & \gcell & \gcell & \gcell &  &  &  &  &  &  & X &  \\
G4 & \gcell & \gcell & \gcell & \gcell & \gcell & X &  &  &  &  &  &  & X \\
G5 & \gcell & \gcell & \gcell & \gcell & \gcell &  & X & X &  &  &  &  &  \\
G6 & \gcell & \gcell & \gcell & \gcell & \gcell & \gcell & \gcell & \gcell & \gcell & \gcell & \gcell & \gcell & \gcell \\
G7 & \gcell & \gcell & \gcell & \gcell & \gcell & \gcell & \gcell & \gcell & \gcell & \gcell & \gcell & \gcell & \gcell \\
\bottomrule
\end{tabular}
\end{table}

We evaluate the \sys agent on the same face-to-face conversational agent application presented in~\Cref{sec:method:case_study}. We provide the agent a sequential baseline, structured like the pseudocode shown in~\Cref{fig:ex_user_base_impl}.

Our results are presented in~\Cref{tab:liveavatar}, and the multi-GPU deployments' node-to-GPU placements are shown in~\Cref{tab:place_liveavatar}. We find these results to be quite impressive:
\begin{itemize}[itemsep=0.0pt,topsep=0pt,leftmargin=*]
    \item With two GPUs, \sys reduces TTFO latency by \textbf{$\sim$70$\times$}. The agent starts S2V inference on intermediate TTS audio chunks as soon as they become available rather than waiting for the full TTS audio, and it streams each generated video chunk to the output buffer while later chunks are still being produced, so the first frame is emitted long before the response has been fully rendered. Its placement co-locates the ASR with the S2V DiT and VAE, and the LLM with the TTS: the ASR completes before the LLM starts, and the LLM completes before the TTS $\to$ S2V $\to$ output pipeline begins, so these components never contend for the same GPU, meaning that disaggregating them does not achieve any additional benefit.
    \item To raise frame rate, the agent expands to six GPUs and applies pipeline parallelism to the S2V model (which was noted as a legal strategy for LiveAvatar S2V in~\Cref{sec:method:case_study}, and which~\Cref{sec:analysis:app2} shows sustains at least as high a frame rate as sequence parallelism when sequence-parallel speedup is sublinear). This provides a \textbf{4.3$\times$} frame rate improvement over the 2-GPU variant with minimal impact to latency.
\end{itemize}
Ultimately, \sys is able to design a highly efficient implementation that reduces latency from the user's sequential baseline by \textbf{$\sim$70$\times$} while sustaining an extremely high theoretical frame rate that allows real-time playability.

\subsection{Generalization to other applications}
\label{sec:exp:generalization}

We evaluate on a set of four additional applications and demonstrate that \sys can generalize to diverse settings while still producing efficient deployments.

\subsubsection{Qwen3-Omni}
\label{sec:exp:qwen3omni}

\begin{table}[t]
\centering
\caption{
Comparison between \sys and the hand-engineered
vLLM-Omni deployment for serving Qwen3-Omni with 3 GPUs.
\sys achieves lower latency while maintaining a Real-Time Factor
(RTF) below 1, enabling real-time audio generation after the
initial response latency.
}
\label{tab:qwen3omni}

\setlength{\tabcolsep}{6pt}
\renewcommand{\arraystretch}{1.15}

\small
\begin{tabular}{l c c c}
\toprule
\textbf{Deployment} &
\textbf{\# GPUs} &
\textbf{Latency (s)} $\boldsymbol{\downarrow}$ &
\textbf{RTF $< 1$} \\
\midrule

Sequential (no streaming)
& 1
& 42.713
& \checkmark \\

vLLM-Omni
& 3
& 0.433
& \checkmark \\

\sys
& 3
& \textbf{0.323}
& \checkmark \\

\bottomrule
\end{tabular}
\end{table}

Qwen3-Omni~\citep{xu2025qwen3omni} is a natively multimodal model that responds to a user prompt through three components: a Thinker LLM that produces a text response, a Talker LLM that converts the response into audio codec tokens, and a Vocoder that synthesizes the codec tokens into audio. vLLM-Omni~\citep{yin2026vllmomni} explicitly implements disaggregation and inter-stage streaming for this model. We provide the \sys agent a sequential baseline that independently serves each component and coordinates them synchronously, and we find that the \sys agent automatically discovers the disaggregation structure and streaming logic on its own.

As shown in~\Cref{tab:qwen3omni}, \sys achieves \textbf{25\% latency reduction} compared to vLLM-Omni. This is achieved by using lighter-weight inter-component data transfer compared to vLLM-Omni's more generalized implementation.~\Cref{tab:qwen3omni} also measures Real-Time Factor (RTF), the ratio between end-to-end (E2E) generation time and length of generated audio; \sys maintains RTF $< 1$, meaning that its output remains real-time while still delivering lower latency.

\subsubsection{Video Background Editor (Krea-Realtime + SAM 3)}
\label{sec:exp:krea_sam3}

In this application, the user provides a continuous webcam video stream. The video is processed in two paths: Krea-Realtime~\citep{millon2025krea}, an autoregressive video model that supports video-to-video (V2V) generation, performs style transfer on the stream; Segment Anything Model (SAM) 3~\citep{carion2026sam3} segments the user's body per frame. The two outputs are combined by replacing the V2V output with the original webcam pixels at all body locations, producing a video where only the background is restyled. The V2V and SAM 3 paths are independent until the final composition step and can therefore execute in parallel; within Krea-Realtime, the DiT can be sequence-parallelized; like WorldPlay, the DiT and VAE can be co-located or disaggregated.

\begin{table}[t]
\centering
\caption{
Comparison of video background editor deployments
discovered by the agent against the sequential baseline.
}
\label{tab:krea_sam3}

\setlength{\tabcolsep}{6pt}
\renewcommand{\arraystretch}{1.15}

\small
\begin{tabular}{l c c c}
\toprule
\textbf{Deployment} &
\textbf{\# GPUs} &
\textbf{Latency (ms)} $\boldsymbol{\downarrow}$ &
\textbf{Frame rate (FPS)} $\boldsymbol{\uparrow}$ \\
\midrule

Baseline (sequential)
& 1
& 1715
& 6.82 \\

\sys (parallel)
& 2
& \textbf{1014}
& \textbf{11.54} \\

\midrule

\sys~(latency-optimized)
& 4
& \textbf{491}
& 17.18 \\

\sys~(frame-rate-optimized)
& 5
& 517
& \textbf{19.41} \\

\bottomrule
\end{tabular}
\end{table}

\begin{table}[t]
\centering
\caption{Placement $\rho$ for the video background editor deployments. The latency-optimized deployment co-locates the DiT and VAE on one SP-3 group, with SAM~3 on a 4th GPU; the frame-rate-optimized one moves the VAE onto its own GPU, taking the total to 5. SAM~3 runs on its own GPU throughout.}
\label{tab:place_krea}
\small
\setlength{\tabcolsep}{4pt}
\renewcommand{\arraystretch}{1.1}
\begin{tabular}{@{}lP{2.9em}P{2.9em}P{2.9em}@{\hspace{0.6em}}!{\vrule}@{\hspace{0.6em}}P{2.9em}P{2.9em}P{2.9em}@{}}
\toprule
& \multicolumn{3}{c}{\textbf{Latency-optimized}} & \multicolumn{3}{c}{\textbf{Frame-rate-optimized}} \\
\cmidrule(lr){2-4} \cmidrule(lr){5-7}
GPU & DiT & VAE & SAM\,3 & DiT & VAE & SAM\,3 \\
\midrule
G0 & X & X &  & X &  &  \\
G1 & X & X &  & X &  &  \\
G2 & X & X &  & X &  &  \\
G3 &  &  & X &  & X &  \\
G4 & \gcell & \gcell & \gcell &  &  & X \\
\bottomrule
\end{tabular}
\end{table}

Results are reported in~\Cref{tab:krea_sam3}, with the placements visualized in~\Cref{tab:place_krea}. Importantly, the agent automatically chooses to deploy SAM 3 on a separate GPU from the V2V model. The agent then produces a latency-optimized and a frame-rate-optimized variant: \textbf{(1) latency-optimized}: the DiT and VAE are co-located with SP-3, reducing latency by \textbf{3.5$\times$} over the baseline; \textbf{(2) frame-rate-optimized}: the agent uses SP-3 for DiT and places the VAE on a separate GPU to enable pipeline parallelism, resulting in \textbf{2.8$\times$} higher frame-rate compared to the baseline.

\subsubsection{Video World Model (WorldPlay)}
\label{sec:exp:worldplay}

WorldPlay is an autoregressive video world model: the user issues keyboard actions (WASD or arrow keys), and the model renders these actions as movement in a generated video stream. We use HY-WorldPlay 5B~\citep{sun2026worldplay}, which decomposes into an autoregressive DiT and a streaming VAE. The DiT and VAE share no persistent state, so they can be disaggregated on separate GPUs and pipelined across batches; alternatively, they can be co-located and sequence-parallelized across a shared group.

\begin{table}[t]
\centering
\caption{
Deployment comparisons for WorldPlay at 1 and 2 GPUs.
The agent leverages disaggregation, co-location, and sequence
parallelism to construct deployments that improve latency
and/or frame rate.
}
\label{tab:worldplay}

\setlength{\tabcolsep}{6pt}
\renewcommand{\arraystretch}{1.15}

\small
\begin{tabular}{c l c c}
\toprule
\textbf{\# GPUs} &
\textbf{Deployment} &
\textbf{Latency (ms)} $\boldsymbol{\downarrow}$ &
\textbf{Frame rate (FPS)} $\boldsymbol{\uparrow}$ \\
\midrule

1 &
Baseline (sequential) &
869 & 20.4 \\

\midrule

\multirow{2}{*}{2}
& \sys~(frame-rate-optimized)
& 625
& \textbf{31.0} \\

& \sys~(latency-optimized)
& \textbf{493}
& 25.5 \\

\bottomrule
\end{tabular}
\end{table}

\begin{table}[t]
\centering
\caption{Placement $\rho:\mathcal V\to\mathcal R$ for the latency- and frame-rate-optimized WorldPlay deployments on 2 GPUs. The latency-optimized deployment co-locates the DiT and VAE on one SP-2 group (short critical path); the frame-rate-optimized one disaggregates them onto separate GPUs and pipelines across batches (same tradeoff shown in~\Cref{sec:analysis:app1}).}
\label{tab:place_worldplay}
\small
\setlength{\tabcolsep}{4pt}
\renewcommand{\arraystretch}{1.1}
\begin{tabular}{@{}lP{3.6em}P{3.6em}@{\hspace{0.6em}}!{\vrule}@{\hspace{0.6em}}P{3.6em}P{3.6em}@{}}
\toprule
& \multicolumn{2}{c}{\textbf{Latency-optimized}} & \multicolumn{2}{c}{\textbf{Frame-rate-optimized}} \\
\cmidrule(lr){2-3} \cmidrule(lr){4-5}
GPU & DiT & VAE & DiT & VAE \\
\midrule
G0 & X & X & X &  \\
G1 & X & X &  & X \\
\bottomrule
\end{tabular}
\end{table}

Results are reported in~\Cref{tab:worldplay}, with the node-to-GPU placements of the two 2-GPU deployments visualized in~\Cref{tab:place_worldplay}. We find that the agent employs both disaggregation of DiT and VAE as well as co-location with sequence parallelism to design deployments that tackle both latency and output frame rate. For higher frame rate, the agent chooses to disaggregate the DiT and VAE: this enables pipelining and executing them in parallel on separate GPUs to reduce inter-batch latency. For lower latency, the agent chooses to instead co-locate the DiT and VAE on the same set of GPUs: this enables a higher SP degree within the same GPU budget to reduce the critical path latency. Notably, this is also the same as reducing the inter-batch latency, so frame rate also improves. These two choices instantiate~\Cref{sec:analysis:app1}, which proves that for an actions $\to$ DiT $\to$ VAE pipeline, co-location minimizes the latency-binding critical path while disaggregation can sustain a higher frame rate.

The agent proposes further variants that scale these strategies to larger GPU budgets, reaching up to \textbf{51\%} lower latency and \textbf{2.2$\times$} the baseline frame rate; we present these deployments and their analysis in~\Cref{sec:supp:worldplay}.

\subsubsection{Video Narrator (LongLive)}
\label{sec:exp:longlive}

In this application, the user narrates a video in real time: the user speaks prompts that an automatic speech recognition (ASR) model transcribes into text, and an autoregressive video model renders a continuous video stream that updates to reflect each new prompt. We use LongLive-2.0-5B~\citep{chen2026longlive2} for video generation, which decomposes into an autoregressive DiT and a streaming VAE. The ASR shares no persistent state with the video model and supplies a new prompt only when the user speaks, so it can be placed on a separate GPU and run asynchronously to overlap transcription with generation. Like WorldPlay, the DiT can be sequence-parallelized, and the DiT and VAE can be co-located and sequence-parallelized or disaggregated onto separate GPUs and pipelined across batches, with the VAE additionally tiled across GPUs.

\begin{table}[!ht]
\centering
\caption{
Deployment comparisons for the LongLive video narrator at 1 and 2 GPUs. At a given GPU budget the
agent offers a latency-optimized deployment (co-located DiT and VAE at a higher
SP degree) and a frame-rate-optimized deployment (DiT and VAE
disaggregated and pipelined across GPUs), trading the two targets against each other.
}
\label{tab:longlive}

\setlength{\tabcolsep}{6pt}
\renewcommand{\arraystretch}{1.15}

\small
\begin{tabular}{c l c c}
\toprule
\textbf{\# GPUs} &
\textbf{Deployment} &
\textbf{Latency (ms)} $\boldsymbol{\downarrow}$ &
\textbf{Frame rate (FPS)} $\boldsymbol{\uparrow}$ \\
\midrule

1 &
Baseline (sequential) &
674 &
25.8 \\

\midrule

\multirow{2}{*}{2}
& \sys~(latency-optimized)
& \textbf{512}
& 36.5 \\

& \sys~(frame-rate-optimized)
& 630
& \textbf{47.8} \\

\bottomrule
\end{tabular}
\end{table}

Results are reported in~\Cref{tab:longlive}. As in WorldPlay, at a given GPU budget the agent offers a latency-optimized and a frame-rate-optimized deployment that trade the two targets against each other. For lower latency it co-locates the DiT and VAE, which affords a higher SP degree within the budget and shortens the serial per-chunk denoising critical path; for higher frame rate it instead disaggregates the DiT and VAE onto separate GPUs so that the VAE decode of one chunk overlaps the DiT computation of the next. At 2 GPUs this is the choice between co-located SP-2 and a disaggregated DiT/VAE pipeline, which trades \textbf{23}\% higher latency for \textbf{31}\% higher frame rate. The agent proposes further variants that scale these strategies to larger GPU budgets, reaching \textbf{1.6$\times$} lower latency and \textbf{2.6$\times$} the baseline frame rate; we present these deployments and their analysis in~\Cref{sec:supp:longlive}.

\begin{table}[H]
\centering
\caption{
Deployments discovered by the \sys agent on AMD MI355X for all five applications, re-optimized from scratch under the same references and GPU budgets used for B200.
}
\label{tab:mi355x}
\small
\setlength{\tabcolsep}{6pt}
\renewcommand{\arraystretch}{1.1}
\begin{subtable}[t]{\linewidth}
\centering
\caption{Face-to-face conversational agent (LiveAvatar)}
\label{tab:mi355x_la}
\begin{tabular}{c l c c}
\toprule
\textbf{\# GPUs} & \textbf{Deployment} & \textbf{Latency (s)} $\boldsymbol{\downarrow}$ & \textbf{Frame rate (FPS)} $\boldsymbol{\uparrow}$ \\
\midrule
1 & Baseline (sequential, no streaming) & 101.17 & -- \\
3 & \sys~(streaming) & 1.59 & 37.76 \\
8 & \sys~(streaming + S2V PP + SP) & \textbf{1.47} & \textbf{78.61} \\
\bottomrule
\end{tabular}
\end{subtable}

\vspace{0.7em}

\begin{subtable}[t]{\linewidth}
\centering
\caption{Multimodal LLM (Qwen3-Omni)}
\label{tab:mi355x_qwen}
\begin{tabular}{c l c c}
\toprule
\textbf{\# GPUs} & \textbf{Deployment} & \textbf{Latency (s)} $\boldsymbol{\downarrow}$ & \textbf{RTF $<1$} \\
\midrule
1 & Sequential (no streaming) & 29.66 & \checkmark \\
3 & vLLM-Omni & 0.779 & \checkmark \\
3 & \sys & \textbf{0.276} & \checkmark \\
\bottomrule
\end{tabular}
\end{subtable}

\vspace{0.7em}

\begin{subtable}[t]{\linewidth}
\centering
\caption{Video background editor (Krea-Realtime + SAM~3)}
\label{tab:mi355x_krea}
\begin{tabular}{c l c c}
\toprule
\textbf{\# GPUs} & \textbf{Deployment} & \textbf{Latency (ms)} $\boldsymbol{\downarrow}$ & \textbf{Frame rate (FPS)} $\boldsymbol{\uparrow}$ \\
\midrule
1 & Baseline (sequential) & 2393 & 5.40 \\
4 & \sys~(latency-optimized) & \textbf{612} & 7.93 \\
5 & \sys~(frame-rate-optimized) & 646 & \textbf{8.90} \\
\bottomrule
\end{tabular}
\end{subtable}

\vspace{0.7em}

\begin{subtable}[t]{\linewidth}
\centering
\caption{Video world model (WorldPlay)}
\label{tab:mi355x_wp}
\begin{tabular}{c l c c}
\toprule
\textbf{\# GPUs} & \textbf{Deployment} & \textbf{Latency (ms)} $\boldsymbol{\downarrow}$ & \textbf{Frame rate (FPS)} $\boldsymbol{\uparrow}$ \\
\midrule
1 & Baseline (sequential) & 1102 & 15.6 \\
\midrule
2 & \sys~(latency-optimized) & \textbf{576} & 19.8 \\
2 & \sys~(frame-rate-optimized) & 723 & \textbf{29.5} \\
\midrule
4 & \sys~(latency-optimized) & \textbf{320} & 31.8 \\
8 & \sys~(frame-rate-optimized) & 326 & \textbf{56.8} \\
\bottomrule
\end{tabular}
\end{subtable}

\vspace{0.7em}

\begin{subtable}[t]{\linewidth}
\centering
\caption{Video narrator (LongLive)}
\label{tab:mi355x_ll}
\begin{tabular}{c l c c}
\toprule
\textbf{\# GPUs} & \textbf{Deployment} & \textbf{Latency (ms)} $\boldsymbol{\downarrow}$ & \textbf{Frame rate (FPS)} $\boldsymbol{\uparrow}$ \\
\midrule
1 & Baseline (sequential) & 784 & 18.9 \\
\midrule
2 & \sys~(latency-optimized) & \textbf{660} & 20.4 \\
2 & \sys~(frame-rate-optimized) & 907 & \textbf{30.5} \\
\midrule
4 & \sys~(latency-optimized) & \textbf{343} & 38.8 \\
8 & \sys~(frame-rate-optimized) & 442 & \textbf{56.5} \\
\bottomrule
\end{tabular}
\end{subtable}
\end{table}

\FloatBarrier

\subsection{Generalization across hardware}
\label{sec:exp:hardware}

The experiments so far run on NVIDIA B200 GPUs. To test whether \sys generalizes across accelerators, we separately execute the full agent pipeline on AMD MI355X for the same set of applications. Other than the hardware, the protocol matches~\Cref{sec:exp:setup} in every respect: the agent is given the same reference implementations, the same GPU budgets, and no hardware-specific hints, and it optimizes each application from scratch (the agent is provided no information on the B200 runs).~\Cref{tab:mi355x} reports the deployments it discovers across all five applications.

On MI355X, \sys reduces latency by up to $\sim$\textbf{70$\times$} and improves throughput by up to \textbf{3.6$\times$}, while beating the hand-engineered vLLM-Omni baseline on Qwen3-Omni by \textbf{65\%}. Across all five applications the agent recovers the same deployment families it found on B200 and reproduces the central latency/frame-rate tradeoff, in which co-location minimizes latency and disaggregation raises frame rate. On WorldPlay, at 2 GPUs co-location reaches the lower latency (576 ms) and disaggregation the higher frame rate (29.5 FPS), the same split the agent makes on B200. On Qwen3-Omni the agent again outperforms the vLLM-Omni deployment (0.276 s versus 0.779 s) while keeping RTF below 1 (\Cref{tab:mi355x_qwen}), meaning the output streams in real-time.

The agent also adapts its choices to the hardware. MI355X's high-degree co-location is strong enough that WorldPlay and LongLive reach \emph{lower} best latency than on B200 (320 ms and 343 ms). On the conversational agent application, the \sys agent selects a different decomposition of the S2V model than it does on B200: per-step pipelining admits only four stages, one per denoising step, and therefore cannot place more than four GPUs on that model; the agent instead merges steps into two pipeline stages and applies SP-3 to each stage, occupying the full 8-GPU budget. Notably, the agent observes that, under a matched streaming configuration, this deployment sustains roughly 30\% higher frame rate than per-step pipelining reaches at six GPUs; composing it with full streaming produces the deployment in~\Cref{tab:mi355x_la} that is best on both metrics. Its streaming deployment also gives the TTS its own GPU instead of sharing one with the LLM; as in~\Cref{sec:exp:liveavatar}, these components never execute concurrently, so this is a difference in placement rather than in performance. The one application whose frame rate does not scale as on B200 is the video background editor, where the pipeline is bound by the SAM~3 segmentation path on MI355X and disaggregating the VAE yields little additional throughput. On MI355X the agent generally produces an efficient deployment for each application and makes the same latency and frame-rate tradeoffs it does on B200.

\FloatBarrier

\subsection{Ablation study}
\label{sec:exp:ablation}

The previous experiments directly evaluate the deployments the agent produces. We now examine the harness itself to identify which of its components are responsible for the results. This study is conducted on an 8-GPU NVIDIA B200 node and uses the face-to-face conversational agent from~\Cref{sec:exp:liveavatar}, the application in our set that admits the richest combination of strategies (chunked TTS $\to$ S2V streaming, streaming of generated video to the output buffer, and pipeline parallelism across the S2V denoising steps).

\begin{table}[t]
\centering
\caption{
Ablation of the two components of the \sys workflow on the face-to-face conversational agent. Each configuration receives the same starting prompt, the same 8-GPU budget, and the same requirement to verify correctness. We report the best deployment each ablation produces. Removing either component results in the agent either overlooking an optimization axis or failing to develop an efficient implementation of it.
}
\label{tab:ablation}
\setlength{\tabcolsep}{4pt}
\renewcommand{\arraystretch}{1.15}
\small
\begin{tabular}{l c c l c c c}
\toprule
\textbf{Configuration} &
\textbf{IR pass} &
\textbf{Loop} &
\textbf{Best deployment} &
\textbf{\# GPUs} &
\textbf{Latency (s)} $\boldsymbol{\downarrow}$ &
\textbf{Frame rate (FPS)} $\boldsymbol{\uparrow}$ \\
\midrule
Reference & -- & -- & Sequential baseline & 1 & 117.10 & -- \\
\midrule
Naive & $\times$ & $\times$ & S2V output streaming & 2 & 10.45 & 30.99 \\
No IR & $\times$ & \checkmark & Full streaming & 2 & 1.68 & 30.67 \\
No Loop & \checkmark & $\times$ & Full streaming + S2V PP & 6 & 2.37 & 44.69 \\
\sys & \checkmark & \checkmark & Full streaming + S2V PP & 6 & \textbf{1.72} & \textbf{133.48} \\
\bottomrule
\end{tabular}
\end{table}

We remove each of the two components of the workflow, the IR pass and the validation loop. The \emph{Naive} configuration removes both and generically instructs the agent to develop optimized deployments; \emph{No IR} keeps only the loop; \emph{No Loop} keeps only the IR pass and then instructs the agent to generically develop efficient deployments; and the full \sys configuration keeps both components.~\Cref{tab:ablation} reports the best deployment produced by each, and~\Cref{sec:supp:analysis:ablation} provides a more detailed analysis of the variants each configuration measured. Notably, removing each of the two \sys components fails in different ways, and the pattern matches the two failure modes identified in~\Cref{sec:method:case_study}:
\begin{itemize}[itemsep=0.0pt,topsep=0pt,leftmargin=*]
    \item \textbf{Without the IR pass, the agent does not surface the pipeline-parallel axis at all.} The Naive configuration finds only that generated video can be streamed to the output buffer, never that the TTS can be streamed into the S2V model, and fails to address low frame rate. Adding the loop (No IR) recovers the missing streaming edge, because iteration keeps the agent exploring after its first success, and the resulting deployment attains the same $\sim$70$\times$ latency reduction as the full system. Neither configuration ever proposes pipeline parallelism, so both remain near 31 FPS. Pipeline parallelism is discovered only in the two configurations that build the IR, where the annotation of per-step KV caches as disjoint persistent state makes the denoising steps visibly independent.
    \item \textbf{Without the loop, the agent surfaces the pipeline-parallel axis but does not turn it into an efficient implementation.} The No Loop configuration finds both full streaming and S2V pipeline parallelism during IR analysis and directly attempts to implement a combined variant. It implements pipeline parallelism by orchestrating the four steps within a single process, so every synchronization event blocks all steps rather than only the steps involved, and it terminates at this prototype implementation with 44.69 FPS. The full \sys system instead isolates each strategy as its own variant before composing them: iterating on the pipeline-parallel deployment produces a multi-process implementation with non-blocking NCCL transfers between steps. Composing that implementation with streaming yields \textbf{3.0$\times$} the frame rate of the No Loop deployment at \textbf{1.4$\times$} lower latency.
\end{itemize}
The two components are therefore complementary: the IR pass determines which optimization axes the agent targets, and the loop determines how well it exploits them and composes them into optimized deployments.

\subsection{Robustness}
\label{sec:exp:robustness}

The deployments reported so far come from a single agent configuration, with only a single run per application. To measure the robustness of the \sys framework, we conduct two additional studies to evaluate modified agent configurations and reliability across independent runs. Both studies below rerun the face-to-face conversational agent of~\Cref{sec:exp:liveavatar} on the 8-GPU NVIDIA B200 node.

\begin{table}[t]
\centering
\caption{
Robustness of the harness on the face-to-face conversational agent application. In both studies, the agent is instructed to optimize the application given a full 8-GPU node.
}
\label{tab:robustness}
\small
\setlength{\tabcolsep}{3pt}
\renewcommand{\arraystretch}{1.15}
\begin{subtable}[t]{\linewidth}
\centering
\caption{Underlying coding agent and model}
\label{tab:robustness_model}
\begin{tabular}{l l c c c}
\toprule
\textbf{Agent} & \textbf{Best deployment} & \textbf{\# GPUs} & \textbf{Latency (s)} $\boldsymbol{\downarrow}$ & \textbf{Frame rate (FPS)} $\boldsymbol{\uparrow}$ \\
\midrule
Claude Code + Opus 4.8 & Full streaming + S2V PP & 6 & 1.72 & 133.48 \\
Codex + GPT-5.6 Sol & Full streaming + S2V PP + VAE SP & 8 & 1.77 & 131.98 \\
\bottomrule
\end{tabular}
\end{subtable}

\vspace{0.7em}

\begin{subtable}[t]{\linewidth}
\centering
\caption{Independent runs of the same configuration}
\label{tab:robustness_repro}
\begin{tabular}{c l c c c}
\toprule
\textbf{Run} & \textbf{Best deployment} & \textbf{\# GPUs} & \textbf{Latency (s)} $\boldsymbol{\downarrow}$ & \textbf{Frame rate (FPS)} $\boldsymbol{\uparrow}$ \\
\midrule
1 & Full streaming + S2V PP & 6 & 1.72 & 133.48 \\
2 & Full streaming + S2V PP & 8 & 1.63 & 133.07 \\
3 & Full streaming + S2V PP & 8 & 1.78 & 125.30 \\
\bottomrule
\end{tabular}
\end{subtable}
\end{table}

\paragraph{Sensitivity to the underlying model.}
To observe the effects of modifying the agent configuration under the same \sys harness, we conduct an additional run to optimize the conversational agent application using OpenAI Codex~\citep{openai2026codexcli} with GPT-5.6 Sol~\citep{openai2026gpt56sol}, under the same prompt and the same 8-GPU budget.~\Cref{tab:robustness_model} shows that the modified agent configuration discovers both key optimizations (full streaming and S2V pipeline parallelism), and reaches essentially the same latency and frame rate metrics. The only observed difference is that GPT-5.6 chooses to shard the VAE with sequence parallelism, so its deployment occupies 8 GPUs instead of 6; this choice changes neither metric, as frame rate is capped by the latency of a single denoising step, which already exceeds that of the single-rank VAE, and the added multi-process coordination only slightly increases latency.~\Cref{sec:supp:analysis:model} provides a more detailed explanation of the other variants explored in this run. Ultimately, the design principles of \sys are not tailored to one coding agent or model.

\paragraph{Consistency across runs.}
To observe the robustness of the agent across independent runs, we repeat the experiment from~\Cref{sec:exp:liveavatar} a total of three times to see if all runs arrive at the same solution.~\Cref{tab:robustness_repro} shows that all three runs converge on the same deployment, i.e., full streaming composed with S2V pipeline parallelism, within 1.63--1.78 s of latency and 125.30--133.48 FPS. In \Cref{sec:supp:analysis:repro}, we provide a more detailed analysis of the variants explored in each run, and we importantly observe that all three runs follow the same exploration path: each proposes a streaming-only variant, then a pipeline-parallel-only variant, then a composition of the two. The runs differ only in placement: runs 2 and 3 disaggregate the ASR, LLM, TTS, and VAE over four GPUs for a total of 8, where run 1 pairs them onto two. This has little effect, since (as noted in~\Cref{sec:exp:liveavatar}) these components never execute concurrently with one another.

%% file: sections/conclusion.tex
\section{Conclusion}
\label{sec:conclusion}

In this paper, we present \sys, an agent harness for using a coding agent to flexibly deploy real-time, multimodal interactive applications. Our agentic solution is motivated by the intractability of deriving efficient deployments from arbitrary application specifications and the promise of recent coding agents at using higher-level reasoning to perform effective code synthesis and optimization. We discover that, by instilling an agent with a properly structured workflow---specifically consisting of an IR conversion step to make data dependencies and streaming behaviour explicit, as well as a self-driven validation loop to systematically encourage diverse and principled agentic exploration---agents are capable of converting simple, human-authored application implementations into highly efficient deployments. Our results demonstrate that \sys is able to flexibly optimize for various targets such as throughput or latency, and the deployments it generates significantly outperform simple handwritten baselines while also being generally on par or more performant than solutions proposed by other serving systems or expert engineers.

\sys is also inexpensive to run: producing each application's deployments costs between \$75 and \$167 in API usage and a few hours of agent time (as discussed further in \Cref{sec:supp:cost}), comparable to or below the expert engineering effort it replaces. This bears on the scalability concern that motivates this work, since manual deployment engineering is bounded by the availability of engineers with the relevant systems expertise, whereas agent-driven deployment is bounded by compute that can be provisioned per application.

A current limitation of our system is that we have not integrated LLM kernel optimization agents. We leave this as future work to enable a broader range of optimization opportunities. Our existing experiments find that the harness can transfer to additional coding agent configurations; however, we do not claim that \sys is model-agnostic: the framework assumes an underlying model capable of the native reasoning needed to organize a hierarchical IR and maintain a hypothesis queue. Studying how \sys behaves under weaker or open-source models with modified agent scaffolding remains future work.

\section{Acknowledgements}

We are grateful to AMD and NVIDIA for access to computing resources. This work was partially supported by Google Research Award, Google ML \& System Junior Faculty Award, Amazon Research Award, Fireworks AI, Intel, Li Auto, Moffett AI, and CMU CyLab Seed funding. This material is also based upon work supported by the National Science Foundation under Grant Nos. CCF-2504353 and CCF-2247014, and by IARPA. Any opinions, findings, conclusions or recommendations expressed are those of the authors and do not necessarily reflect the views of the National Science Foundation.

%% file: sections/supp.tex
\section{Formal Analysis of Pipeline Graphs}
\label{sec:analysis}

Here we provide formalizations for pipeline graphs to properly analyze various deployment strategies on a few applications. Although our analysis can extend to more general applications, we focus on applications with streaming generation so we can relate our analysis to applications we performed empirical analysis against in~\Cref{sec:exp}.

\subsection{Setup}
\label{sec:analysis:setup}

We introduce a general framework for formally defining applications as pipeline graphs with explicit deployment considerations to assist with the later analysis.

\subsubsection{Initial assumptions}

We take some initial assumptions to simplify the downstream analysis. Assume a single request produces an indefinitely long stream of batches indexed by $j\ge 0$. Users emit $N$ actions per batch period $T$, one action per subinterval of length $\Delta=\frac{T}{N}$. Action $jN+k$ arrives during
\[
[jT+k\Delta,\;jT+(k+1)\Delta)
\]
and becomes chunk $k$ of batch $j$. We conservatively take batch $j$ to be ready at
\[
R_j=(j+1)T.
\]

The output-rate is constant, so chunk $k$ of batch $j$ is output at
\[
d_{j,k}\;=\;d_0+jT+k\Delta,
\]
where $d_0$ is a single offset, fixed once at the start of the stream and shared across every batch for the request. A smaller $d_0$ corresponds to lower latency, so the objective is to make the display schedule as tight as possible while remaining feasible across the entire stream of batches.

\subsubsection{Pipeline graph}

Each batch is processed by a fixed collection of operations (e.g., a DiT followed by a streaming VAE) that together turn the batch's $N$ actions into $N$ display-ready output chunks. We describe this processing pipeline by a static template plus a deployment plan. The template, independent of how we deploy, has:
\begin{itemize}
\item $\mathcal V$, a finite set of operation types, each running once per batch (so the instances are pairs $(j,v)$ for $j\ge 0$, and a single batch's processing comprises all $|\mathcal V|$ instances $\{(j,v):v\in\mathcal V\}$);
\item $\mathcal R$, a finite set of resources (a single GPU, or a multi-GPU group treated as one resource);
\item a set $E\subseteq\mathcal V\times\mathcal V\times\{0,1\}$ of logical edges. Each edge $(u,v,\lambda)\in E$ has an \emph{edge batch-shift} $\lambda\in\{0,1\}$ and means that $(j,v)$ cannot start until $(j-\lambda,u)$ finishes, with $\lambda=0$ encoding an intra-batch dependency and $\lambda=1$ encoding a dependency on the previous batch's instance;
\item a subset $\mathcal V_{\mathrm{in}}\subseteq\mathcal V$ of \emph{entry} operations gated by action arrival, meaning that no instance $(j,v)$ with $v\in\mathcal V_{\mathrm{in}}$ may start in batch $j$ before time $R_j$;
\item an output map $o:\{0,\dots,N-1\}\to\mathcal V$ identifying the operation whose completion releases each chunk $k$.
\end{itemize}
The deployment plan specifies:
\begin{itemize}
\item an assignment $\rho:\mathcal V\to\mathcal R$;
\item for each resource $r\in\mathcal R$, a nonpreemptive time-indexed schedule
\[
A_r:\mathbb R_{\ge 0}\to \{(j,v):j\ge 0,\ \rho(v)=r\}\cup\{\varnothing\},
\]
where $A_r(t)=(j,v)$ means that resource $r$ is executing operation instance $(j,v)$ at time $t$, and $A_r(t)=\varnothing$ means that $r$ is idle.
\end{itemize}
When a resource follows a fixed cyclic per-batch order, we write
$\sigma_r=(v_{r,1},\dots,v_{r,m_r})$ as shorthand for the corresponding time-indexed schedule $A_r$. In that special case, resource $r$ executes $v_{r,1},\dots,v_{r,m_r}$ for batch $j$ before beginning batch $j+1$.

Given the deployment, the deployment-dependent timing data is:
\begin{itemize}
\item a processing time $\tau_v(\rho)$ for each operation, allowed to depend on $\rho$ (e.g.\ a sequence-parallel implementation runs faster on a multi-GPU resource);
\item a synchronization lag $\kappa_{u,v,\lambda}(\rho)\ge 0$ on each edge, also $\rho$-dependent (e.g., an inter-GPU transfer is non-zero only when the two operations are assigned to different resources).
\end{itemize}
For each operation instance $(j,v)$ with $j\ge 0$ and $v\in\mathcal V$, let $S_{j,v}$ denote its start time and 
\[
C_{j,v}:=S_{j,v}+\tau_v(\rho)
\]
its completion time. A deployment
$(\rho,A)$ is feasible if, for every instance $(j,v)$, the schedule satisfies
\[
A_{\rho(v)}(t)=(j,v)\qquad\text{for all }t\in [S_{j,v},C_{j,v}),
\]
and $A_{\rho(v)}(t)\neq (j,v)$ outside this interval. Thus each operation runs nonpreemptively for exactly $\tau_v(\rho)$ time on its assigned resource, and no resource executes more than one operation instance at a time. In addition, all entry and logical-precedence constraints must hold:
\[
S_{j,v}\ge R_j\qquad\text{for all }v\in\mathcal V_{\mathrm{in}},
\]
and, for every edge $(u,v,\lambda)\in E$,
\[
S_{j,v}\ge C_{j-\lambda,u}+\kappa_{u,v,\lambda}(\rho)
\]
whenever the predecessor instance $(j-\lambda,u)$ exists. A cross-batch ($\lambda=1$) edge encodes a logical inter-batch dependency, for instance a shared KV cache that batch $j+1$ inherits from batch $j$. Its $\kappa$ captures any cost of moving that state between resources.

\subsubsection{Latency metric}

Chunk $(j,k)$ becomes display-ready at $Y_{j,k}:=C_{j,o(k)}$, the completion time of $o(k)$ in batch $j$.
\begin{lemma}[Display-offset characterization]\label{lem:display-offset}
Fix a feasible deployment $(\rho,A)$. The display schedule with offset $d_0$
is feasible if and only if
\[
d_0\ge Y_{j,k}-jT-k\Delta
\]
for every $(j,k)$. Consequently, the smallest feasible offset for the deployment is
\[
d_0^*(\rho,A)\;:=\;\sup_{j,k}\bigl(Y_{j,k}-jT-k\Delta\bigr).
\]
\end{lemma}

\begin{proof}
Chunk $(j,k)$ is scheduled to be displayed at time $d_0+jT+k\Delta$ and is
display-ready at time $Y_{j,k}$. Thus the schedule is feasible exactly when
\[
d_0+jT+k\Delta\ge Y_{j,k}
\]
for every $(j,k)$, which is equivalent to
\[
d_0\ge Y_{j,k}-jT-k\Delta
\]
for every $(j,k)$. Taking the supremum over all chunks gives the smallest
feasible offset.
\end{proof}
We use $d_0^*(\rho,A)$ as the action-to-visible latency metric for a deployment.
When the goal is latency optimization, one seeks deployments with smaller
$d_0^*(\rho,A)$. Below we will write $d_0^*(\rho)$ as shorthand for
$d_0^*(\rho,A)$ when the schedule $A$ is determined by context, for example by
the earliest-start schedule or by a fixed cyclic order $\sigma_r$.

\subsubsection{Two key bounds}

\begin{lemma}[Path bound]\label{lem:path-bound}
Fix a feasible deployment $(\rho,A)$. For any logical path $P:v_0\to v_1\to\cdots\to v_m$ in $E$ from an entry operation $v_0\in\mathcal V_{\mathrm{in}}$ to a chunk-$k$ output $v_m=o(k)$, write $\lambda_i\in\{0,1\}$ for the edge batch-shift on the $i$-th edge of $P$ and define the path's \emph{total backward batch shift}
\[
\Lambda(P):=\sum_{i=1}^{m}\lambda_i
\]
(the number of batches by which $P$'s entry instance precedes its output instance). Define the path's total weight
\[
w(P,\rho)\;:=\;\sum_{i=0}^{m}\tau_{v_i}(\rho)+\sum_{i=1}^{m}\kappa_{v_{i-1},v_i,\lambda_i}(\rho).
\]
Then, the smallest feasible display offset $d_0^*$ satisfies
\[
d_0^*\;\ge\;w(P,\rho)+(1-\Lambda(P))T-k\Delta.
\]
\end{lemma}

\begin{proof}
Traversing the edges of $P$ from $v_m$ back to $v_0$, each edge $(v_{i-1},v_i,\lambda_i)$ shifts the batch index back by $\lambda_i$. Hence if $v_m=o(k)$ is at batch $j$, then $v_i$ is at batch $j_i:=j-\sum_{q=i+1}^{m}\lambda_q$. In particular, $v_0$ is at batch $j_0=j-\Lambda(P)$. The entry constraint gives $S_{j_0,v_0}\ge R_{j_0}$, while the per-operation durations and per-edge precedence relations are
\[
C_{j_i,v_i}=S_{j_i,v_i}+\tau_{v_i}(\rho),\qquad
S_{j_{i+1},v_{i+1}}\ge C_{j_i,v_i}+\kappa_{v_i,v_{i+1},\lambda_{i+1}}(\rho).
\]
Telescoping these along $P$ yields
\[
C_{j,o(k)}\;=\;C_{j_m,v_m}\;\ge\;S_{j_0,v_0}+\sum_{i=0}^{m}\tau_{v_i}(\rho)+\sum_{i=1}^{m}\kappa_{v_{i-1},v_i,\lambda_i}(\rho)\;\ge\;R_{j_0}+w(P,\rho).
\]
Since $R_{j_0}=(j_0+1)T=(j+1-\Lambda(P))T$, combining this with~\Cref{lem:display-offset} and using $Y_{j,k}=C_{j,o(k)}$ gives
\begin{align*}
d_0^*\;&\ge\;Y_{j,k}-jT-k\Delta\\
&\ge\;w(P,\rho)+(j+1-\Lambda(P))T-jT-k\Delta\\
&=\;w(P,\rho)+(1-\Lambda(P))T-k\Delta.\qedhere
\end{align*}
\end{proof}

Intuitively, the bound implies that the latency-binding paths are those with high weight $w(P,\rho)$, low batch shift $\Lambda(P)$, and small chunk index $k$, which corresponds to the tightest display deadline.

For throughput analysis, we instead ask how small the inter-batch period $T$ can
be while maintaining finite display offset. This sustainable-period metric is
captured by the following resource-load bound.

\begin{lemma}[Resource-load bound]\label{lem:resource-load}
Recall that $T$ is the batch period, i.e., the duration between receiving
consecutive batches. Any deployment with finite display offset must satisfy
\[
T\;\ge\;T_{\min}(\rho)\;:=\;\max_{r\in\mathcal R}W_r(\rho),
\qquad
W_r(\rho)\;:=\;\sum_{v:\rho(v)=r}\tau_v(\rho).
\]
\end{lemma}

\begin{proof}
Fix a resource $r$. Each batch requires
\[
W_r(\rho)=\sum_{v:\rho(v)=r}\tau_v(\rho)
\]
units of processing time on resource $r$.

Assume the deployment has finite display offset. Then there exists a constant
$B<\infty$ such that the output-relevant work for batch $j$ completes no later
than $jT+B$. In particular, the work on resource $r$ for the first $J$ batches
must be completed by time at most $(J-1)T+B$.

The first $J$ batches require $J W_r(\rho)$ units of processing on resource
$r$. Since resource $r$ can perform at most one unit of work per unit time,
we must have
\[
J W_r(\rho) \le (J-1)T+B.
\]
Dividing by $J$ gives
\[
W_r(\rho) \le \left(1-\frac{1}{J}\right)T+\frac{B}{J}.
\]
Taking $J\to\infty$ yields
\[
W_r(\rho)\le T.
\]
Since this holds for every resource $r$,
\[
T\ge \max_{r\in\mathcal R}W_r(\rho)=T_{\min}(\rho).
\]
\end{proof}

Equivalently, $T_{\min}(\rho)$ is the deployment's sustainable-period metric,
and $\frac{N}{T_{\min}(\rho)}$ is its peak displayed-frame rate.
Thus smaller $d_0^*(\rho,A)$ corresponds to better latency, while smaller
$T_{\min}(\rho)$ corresponds to better throughput.

\subsection{NP-hardness of deployment search}
\label{sec:analysis:np_hard}

We briefly instantiate the framework above to show that the deployment search
problem already contains a classical NP-hard machine-scheduling problem as a
special case. We reduce from identical parallel-machine makespan minimization,
denoted $P||C_{\max}$ in the standard three-field scheduling notation: $P$
indicates identical parallel machines, the empty middle field indicates that
there are no additional job constraints, and $C_{\max}$ is the makespan objective.
In the decision version, the input consists of jobs $1,\dots,n$, processing times
$p_1,\dots,p_n$, $m$ identical machines, and a makespan threshold $B$; the goal
is to decide whether there exists a non-preemptive schedule with makespan at
most $B$. This problem is NP-hard~\citep{lenstra1977complexity}.

Given such an instance, construct a deployment-search instance with
\[
\mathcal V=\{v_1,\dots,v_n\},
\qquad
\mathcal R=\{r_1,\dots,r_m\},
\qquad
E=\varnothing .
\]
Each operation $v_i$ corresponds to job $i$, and each resource $r_\ell$
corresponds to machine $\ell$. The execution time of each operation is
placement-independent:
\[
\tau_{v_i}(\rho)=p_i
\]
for every placement $\rho$. Since there are no logical edges, there are no data
dependencies, state dependencies, or synchronization costs. We may take all
operations to be entry operations, $\mathcal V_{\mathrm{in}}=\mathcal V$; the
output map is arbitrary, since this reduction concerns the sustainable-period
objective.

For any placement $\rho:\mathcal V\to\mathcal R$, the per-batch load on resource
$r$ is
\[
W_r(\rho)
=
\sum_{v_i:\rho(v_i)=r} p_i .
\]
Thus the sustainable period from~\Cref{lem:resource-load} is
\[
T_{\min}(\rho)
=
\max_{r\in\mathcal R} W_r(\rho)
=
\max_{r\in\mathcal R}
\sum_{v_i:\rho(v_i)=r} p_i .
\]
This quantity is exactly the makespan of the corresponding assignment of jobs to
machines. Therefore, there exists a deployment with
\[
T_{\min}(\rho)\le B
\]
if and only if the original $P||C_{\max}$ instance admits a non-preemptive
schedule with makespan at most $B$.

Consequently, even this restricted version of deployment search (with no data
dependencies, no state dependencies, no communication costs,
placement-independent execution times, and only non-preemptive resource capacity
constraints) is NP-hard. The general multimodal deployment problem strictly
contains this special case, because it additionally allows dependency edges,
placement-dependent execution times, synchronization costs, co-location,
disaggregation, streaming, batching, and intra-model parallelism; as such, the general multimodal deployment problem must also be NP-hard.

\subsection{Application 1: actions $\to$ DiT $\to$ VAE}
\label{sec:analysis:app1}
We analyze a two-stage video-generation pipeline consisting of a DiT stage followed by a VAE stage, the structure shared by the WorldPlay, Krea-Realtime, and LongLive applications evaluated in~\Cref{sec:exp:worldplay,sec:exp:krea_sam3,sec:exp:longlive}. We compare two deployment families under a fixed budget of $G$ GPUs: a co-located sequence-parallel deployment, where both stages share one resource, and a disaggregated pipeline-parallel deployment, where DiT and VAE run on separate resources. We evaluate both action-to-visible latency (measured by $d_0^*$), and throughput (measured by $T_{\min}$). We show that co-location is optimal for action-to-visible latency, while disaggregation is optimal for throughput when its best split has no larger sustainable period than the co-located deployment.

\paragraph{Template.}
The static template has:
\begin{itemize}
\item operations $\mathcal V=\{D,V\}$, where $D$ is the DiT operation for one batch and $V$ is the VAE decode operation;
\item resources $\mathcal R$ introduced per deployment below;
\item edges $E=\{(D,V,0),\;(D,D,1),\;(V,V,1)\}$;
\item entry $\mathcal V_{\mathrm{in}}=\{D\}$;
\item output map $o(k)=V$ for every chunk $k\in\{0,\dots,N-1\}$.
\end{itemize}
The edge $(D,V,0)$ is the intra-batch DiT-to-VAE dependency. The inter-batch self-edges $(D,D,1)$ and $(V,V,1)$ encode stage-local ordering or recurrent-state dependencies across consecutive batches.

\paragraph{Deployment plan.}
We compare the following deployment families.

In the co-located deployment $\rho_{\mathrm{col}}$, both stages are assigned to the same $G$-GPU resource $r_{\mathrm{col}}$:
\[
\rho_{\mathrm{col}}(D)=\rho_{\mathrm{col}}(V)=r_{\mathrm{col}},
\qquad
\sigma_{r_{\mathrm{col}}}=(D,V).
\]
Thus the shared resource executes DiT and then VAE for batch $j$ before beginning DiT for batch $j+1$. Since the two stages are co-located,
\[
\kappa_{D,V,0}(\rho_{\mathrm{col}})=0.
\]

In a disaggregated deployment $\rho_{\mathrm{pipe}}^{X,Y}$, DiT is assigned to an $X$-GPU resource and VAE is assigned to a disjoint $Y$-GPU resource, where $X+Y=G$ and $X,Y\ge 1$:
\[
\rho_{\mathrm{pipe}}^{X,Y}(D)=r_D^X,
\qquad
\rho_{\mathrm{pipe}}^{X,Y}(V)=r_V^Y,
\]
with cyclic schedules
\[
\sigma_{r_D^X}=(D),
\qquad
\sigma_{r_V^Y}=(V).
\]
Thus the DiT and VAE resources can process different batches concurrently in steady state.

\begin{proposition}[Co-location is latency-optimal for DiT--VAE]
\label{prop:dit-vae-latency}
For the DiT--VAE deployments above, assume the requested period $T$ is
sustainable and that assigning a stage to the full co-located resource does not
make it slower than assigning that stage to a strict subset of the GPUs:
\[
\tau_D(\rho_{\mathrm{col}})
\le
\tau_D(\rho_{\mathrm{pipe}}^{X,Y}),
\qquad
\tau_V(\rho_{\mathrm{col}})
\le
\tau_V(\rho_{\mathrm{pipe}}^{X,Y})
\]
for every split $X+Y=G$. Then
\[
d_0^*(\rho_{\mathrm{col}})
\le
d_0^*(\rho_{\mathrm{pipe}}^{X,Y})
\qquad
\text{for every split }X+Y=G.
\]
\end{proposition}

\begin{proof}
For each chunk $k$, the direct logical path from the entry to the output is
\[
P_k:\quad D\to V.
\]
This path has $\Lambda(P_k)=0$ and weight
\[
w(P_k,\rho)
=
\tau_D(\rho)+\kappa_{D,V,0}(\rho)+\tau_V(\rho).
\]
By~\Cref{lem:display-offset,lem:path-bound},
\[
d_0^*(\rho)
\ge
T+\tau_D(\rho)+\kappa_{D,V,0}(\rho)+\tau_V(\rho)-k\Delta.
\]
The tightest constraint is for $k=0$, and for the deployments considered here
this lower bound is achieved by the earliest-start schedule whenever $T$ is
sustainable. Therefore
\[
d_0^*(\rho)
=
T+\tau_D(\rho)+\kappa_{D,V,0}(\rho)+\tau_V(\rho).
\]
For co-location, $\kappa_{D,V,0}(\rho_{\mathrm{col}})=0$, so
\[
d_0^*(\rho_{\mathrm{col}})
=
T+\tau_D(\rho_{\mathrm{col}})+\tau_V(\rho_{\mathrm{col}}).
\]
For a disaggregated split,
\[
d_0^*(\rho_{\mathrm{pipe}}^{X,Y})
=
T+\tau_D(\rho_{\mathrm{pipe}}^{X,Y})
+\kappa_{D,V,0}(\rho_{\mathrm{pipe}}^{X,Y})
+\tau_V(\rho_{\mathrm{pipe}}^{X,Y}).
\]
Since synchronization lags are nonnegative and the co-located resource does not
make either stage slower, subtracting the two expressions gives
\[
d_0^*(\rho_{\mathrm{pipe}}^{X,Y})
-
d_0^*(\rho_{\mathrm{col}})
\ge 0.
\]
Thus co-location minimizes the latency-binding path $D\to V$ within the
deployment family considered here.
\end{proof}

\begin{proposition}[Disaggregation can be throughput-optimal for DiT--VAE]
\label{prop:dit-vae-throughput}
For the DiT--VAE deployments above, the co-located deployment has sustainable
period
\[
T_{\min}(\rho_{\mathrm{col}})
=
\max\left\{
\tau_D(\rho_{\mathrm{col}})+\tau_V(\rho_{\mathrm{col}}),\;
\tau_D(\rho_{\mathrm{col}})+\kappa_{D,D,1}(\rho_{\mathrm{col}}),\;
\tau_V(\rho_{\mathrm{col}})+\kappa_{V,V,1}(\rho_{\mathrm{col}})
\right\},
\]
whereas a disaggregated split has sustainable period
\[
T_{\min}(\rho_{\mathrm{pipe}}^{X,Y})
=
\max\left\{
\tau_D(\rho_{\mathrm{pipe}}^{X,Y})+\kappa_{D,D,1}(\rho_{\mathrm{pipe}}^{X,Y}),\;
\tau_V(\rho_{\mathrm{pipe}}^{X,Y})+\kappa_{V,V,1}(\rho_{\mathrm{pipe}}^{X,Y})
\right\}.
\]
Consequently, any split
\[
(X^*,Y^*)
\in
\argmin_{X+Y=G,\ X,Y\ge 1}
T_{\min}(\rho_{\mathrm{pipe}}^{X,Y})
\]
is throughput-optimal among disaggregated deployments. If
\[
T_{\min}(\rho_{\mathrm{pipe}}^{X^*,Y^*})
\le
T_{\min}(\rho_{\mathrm{col}}),
\]
then it is throughput-optimal among the compared co-located and disaggregated
deployments.
\end{proposition}

\begin{proof}
For the co-located deployment, the shared resource must execute both operations
once per batch. By~\Cref{lem:resource-load}, this gives
\[
T\ge \tau_D(\rho_{\mathrm{col}})+\tau_V(\rho_{\mathrm{col}}).
\]
The inter-batch self-edges additionally require
\[
T\ge
\tau_D(\rho_{\mathrm{col}})+\kappa_{D,D,1}(\rho_{\mathrm{col}}),
\qquad
T\ge
\tau_V(\rho_{\mathrm{col}})+\kappa_{V,V,1}(\rho_{\mathrm{col}}),
\]
which yields the stated expression for $T_{\min}(\rho_{\mathrm{col}})$.

For a disaggregated split, the two stages execute on separate resources.
Applying~\Cref{lem:resource-load} to each resource, together with the
inter-batch self-edges, yields the stated expression for
$T_{\min}(\rho_{\mathrm{pipe}}^{X,Y})$. The intra-batch DiT-to-VAE lag
$\kappa_{D,V,0}(\rho_{\mathrm{pipe}}^{X,Y})$ affects latency because it lies on
the path $D\to V$; it affects throughput only if the communication fabric is
modeled as an additional bottleneck resource.

The optimality claim follows directly from the definition of $(X^*,Y^*)$ and
from comparing its sustainable period against that of co-location. If the
inequality is strict, then the disaggregated deployment achieves a strictly
higher peak output rate:
\[
\frac{N}{T_{\min}(\rho_{\mathrm{pipe}}^{X^*,Y^*})}
>
\frac{N}{T_{\min}(\rho_{\mathrm{col}})}.
\]
\end{proof}

\paragraph{Conclusion.}
Propositions~\ref{prop:dit-vae-latency} and~\ref{prop:dit-vae-throughput}
capture a latency-throughput tradeoff: co-location minimizes the latency-binding path $D\to V$, while disaggregation can reduce the sustainable period by allowing DiT and VAE work from different batches to run concurrently.

\subsection{Application 2: sequence parallelism vs stage pipeline parallelism}
\label{sec:analysis:app2}

We now compare sequence parallelism (SP) and stage pipeline parallelism (PP) for
DiT architectures that admit pipeline parallelism across denoising steps, such as the LiveAvatar S2V model deployed in~\Cref{sec:method:case_study,sec:exp:liveavatar}. The key
architectural assumption is that each denoising step maintains independent
per-step KV cache, thus allowing step $Q_i$ of batch $j+1$ to depend only on
step $Q_i$ of batch $j$ rather than on the completion of the entire denoising
chain for batch $j$. Under this assumption, different denoising steps can be
assigned to different resources and executed as a stage pipeline across batches.

Here, we do not prove global optimality over all possible deployments, but we compare SP and PP under the two metrics introduced above: action-to-visible latency (measured by $d_0^*$) and throughput (measured by $T_{\min}$).

\paragraph{Template.}
The static template has:
\begin{itemize}
\item operations $\mathcal V=\{Q_1,\dots,Q_n\}$, where $Q_i$ is the $i$-th
denoising step;
\item resources $\mathcal R$ introduced per deployment below;
\item edges $E = \{(Q_i,Q_{i+1},0):i=1,\dots,n-1\} \cup \{(Q_i,Q_i,1):i=1,\dots,n\}$;
\item entry $\mathcal V_{\mathrm{in}}=\{Q_1\}$;
\item output map $o(k)=Q_n$ for every chunk $k\in\{0,\dots,N-1\}$.
\end{itemize}
The intra-batch edges $(Q_i,Q_{i+1},0)$ encode the denoising chain within a
batch. The inter-batch edges $(Q_i,Q_i,1)$ encode the per-step KV cache dependency enables stage pipeline parallelism.

Let $T_Q$ denote the single-GPU compute time for one denoising step. For an
$n$-GPU sequence-parallel group, let $T_Q^{SP}(n)$ denote the wall-clock time
for one denoising step under sequence parallelism, including any required computation,
communication, and synchronization. We assume
\[
0<T_Q^{SP}(n)\le T_Q,
\]
so sequence parallelism does not make an individual denoising step slower than
running it on one GPU. We do not assume exact linear speedup and instead assume
\[
T_Q\le nT_Q^{SP}(n),
\]
i.e., sublinear sequence-parallel speedup.

\paragraph{Deployment plan.}
We compare two deployments using $n$ GPUs.

In the SP deployment $\rho_{SP}$, all denoising steps are assigned to one
$n$-GPU SP resource $r_n^{SP}$:
\[
\rho_{SP}(Q_i)=r_n^{SP}
\qquad
\text{for every }i.
\]
The resource follows the cyclic schedule
\[
\sigma_{r_n^{SP}}=(Q_1,\dots,Q_n),
\]
so it executes all denoising steps of batch $j$ before beginning the DiT work for
batch $j+1$. The per-step time is
\[
\tau_{Q_i}(\rho_{SP})=T_Q^{SP}(n).
\]
Since all denoising steps are co-located,
\[
\kappa_{Q_i,Q_{i+1},0}(\rho_{SP})=0.
\]

In the PP deployment $\rho_{PP}$, each denoising step is assigned to its own
dedicated GPU:
\[
\rho_{PP}(Q_i)=r_i,
\qquad
r_i\neq r_{i'}\text{ for }i\neq i'.
\]
Each resource $r_i$ follows the cyclic schedule
\[
\sigma_{r_i}=(Q_i),
\]
executing step $Q_i$ for the next batches whenever logical predecessors complete. The per-step time is
\[
\tau_{Q_i}(\rho_{PP})=T_Q.
\]
Since consecutive denoising steps are assigned to different resources, the
intra-batch edges may have nonzero synchronization lags
\[
\kappa_{Q_i,Q_{i+1},0}(\rho_{PP})\ge 0.
\]

In both deployments, each $Q_i$ remains on the same resource across batches, so
the per-step KV cache does not move across resources:
\[
\kappa_{Q_i,Q_i,1}(\rho_{SP})
=
\kappa_{Q_i,Q_i,1}(\rho_{PP})
=
0.
\]

\begin{proposition}[SP has no higher first-frame latency than stage PP]
\label{prop:sp-pp-latency}
For the SP and PP deployments above, assume the requested period $T$ is
sustainable and
\[
T_Q^{SP}(n)\le T_Q.
\]
Then
\[
d_0^*(\rho_{SP})
\le
d_0^*(\rho_{PP}).
\]
\end{proposition}

\begin{proof}
For each chunk $k$, the direct logical path from the entry to the output is
\[
P_k:\quad Q_1\to Q_2\to\cdots\to Q_n.
\]
This path has $\Lambda(P_k)=0$. By~\Cref{lem:display-offset,lem:path-bound},
\[
d_0^*(\rho)
\ge
T+w(P_k,\rho)-k\Delta.
\]
The tightest constraint is for $k=0$, and for the deployments considered here
this lower bound is achieved by the earliest-start schedule whenever $T$ is
sustainable.

For SP,
\[
w(P_0,\rho_{SP})
=
nT_Q^{SP}(n),
\]
so
\[
d_0^*(\rho_{SP})
=
T+nT_Q^{SP}(n).
\]
For PP,
\[
w(P_0,\rho_{PP})
=
nT_Q+\sum_{i=1}^{n-1}\kappa_{Q_i,Q_{i+1},0}(\rho_{PP}),
\]
so
\[
d_0^*(\rho_{PP})
=
T+nT_Q+\sum_{i=1}^{n-1}\kappa_{Q_i,Q_{i+1},0}(\rho_{PP}).
\]
Therefore
\[
d_0^*(\rho_{PP})-d_0^*(\rho_{SP})
=
n\bigl(T_Q-T_Q^{SP}(n)\bigr)
+
\sum_{i=1}^{n-1}\kappa_{Q_i,Q_{i+1},0}(\rho_{PP})
\ge 0.
\]
Thus SP has no higher first-frame latency than stage PP under the stated
assumptions.
\end{proof}

\begin{proposition}[Stage PP has no larger sustainable period under sublinear SP speedup]
\label{prop:sp-pp-throughput}
For the SP and PP deployments above,
\[
T_{\min}(\rho_{SP})=nT_Q^{SP}(n),
\qquad
T_{\min}(\rho_{PP})=T_Q.
\]
Consequently, under the throughput assumption
\[
T_Q\le nT_Q^{SP}(n),
\]
we have
\[
T_{\min}(\rho_{PP})
\le
T_{\min}(\rho_{SP}).
\]
\end{proposition}

\begin{proof}
For SP, the shared resource must execute all $n$ denoising steps once per
batch. By~\Cref{lem:resource-load},
\[
T
\ge
\sum_{i=1}^{n}T_Q^{SP}(n)
=
nT_Q^{SP}(n).
\]
The inter-batch self-edges have zero synchronization lag in this deployment, and
resource non-overlap already enforces the required per-step ordering. Therefore
\[
T_{\min}(\rho_{SP})
=
nT_Q^{SP}(n).
\]

For PP, each resource executes one denoising step per batch. By
\Cref{lem:resource-load},
\[
T\ge T_Q
\]
for every stage resource $r_i$. The inter-batch self-edges again have zero
synchronization lag because each step remains on the same resource across
batches. Therefore
\[
T_{\min}(\rho_{PP})
=
T_Q.
\]
The claimed ordering follows from $T_Q\le nT_Q^{SP}(n)$. If the inequality is
strict, then PP achieves a strictly higher peak frame rate:
\[
\frac{N}{T_Q}
>
\frac{N}{nT_Q^{SP}(n)}.
\]
The intra-batch synchronization lags
$\kappa_{Q_i,Q_{i+1},0}(\rho_{PP})$ affect latency because they lie on the path
$Q_1\to\cdots\to Q_n$; they affect throughput only if the communication fabric
is modeled as an additional bottleneck resource.
\end{proof}

\paragraph{Conclusion.}
Propositions~\ref{prop:sp-pp-latency} and~\ref{prop:sp-pp-throughput}
capture the SP--PP tradeoff. SP reduces the weight of the latency-binding path
$Q_1\to\cdots\to Q_n$, while PP reduces the sustainable period by distributing
the denoising steps across resources. For DiTs without the per-step inter-batch
dependency structure above, the PP schedule analyzed here is not feasible under
the stated framework.

\section{Extended Deployment Results and Scaling Analysis}
\label{sec:supp:extended}

In \Cref{sec:exp}, our experiments present latency- and frame-rate-optimized deployment variants for each application across a range of GPU budgets. This appendix provides extended results for two of these applications, the video world model (WorldPlay) of \Cref{sec:exp:worldplay} and the video narrator (LongLive) of \Cref{sec:exp:longlive}. For each application, we report additional deployments the \sys agent found and analyze in depth the factors that drive how latency and frame rate scale with the number of GPUs.   

\subsection{Video World Model (WorldPlay)}
\label{sec:supp:worldplay}

We use the WorldPlay video world model of \Cref{sec:exp:worldplay}, in which the user's keyboard actions drive an autoregressive DiT and a streaming VAE that render a continuous video stream. \Cref{tab:worldplay_full} reports the full set of deployments the agent found across GPU budgets. At each budget the agent offers a latency-optimized deployment that co-locates the DiT and VAE at a higher SP degree and a frame-rate-optimized deployment that disaggregates them onto separate GPUs so they can be pipelined, and at 6 GPUs it combines the two. \Cref{tab:place_worldplay_full} shows how the disaggregated deployment scales from 2 to 6 GPUs. We analyze how these choices trade latency against frame rate as the budget grows.

\begin{table}[t]
\centering
\caption{
Full set of WorldPlay video-world-model deployments found by the \sys agent across GPU budgets (cf.\ \Cref{tab:worldplay}). The 4- and 6-GPU deployments are analyzed in~\Cref{sec:supp:worldplay}.
}
\label{tab:worldplay_full}

\setlength{\tabcolsep}{6pt}
\renewcommand{\arraystretch}{1.15}

\small
\begin{tabular}{c l c c}
\toprule
\textbf{\# GPUs} &
\textbf{Deployment} &
\textbf{Latency (ms)} $\boldsymbol{\downarrow}$ &
\textbf{Frame rate (FPS)} $\boldsymbol{\uparrow}$ \\
\midrule

1 &
Baseline (sequential) &
869 &
20.4 \\

\midrule

\multirow{2}{*}{2}
& \sys~(frame-rate-optimized)
& 625
& \textbf{31.0} \\

& \sys~(latency-optimized)
& \textbf{493}
& 25.5 \\

\midrule

\multirow{2}{*}{4}
& \sys~(frame-rate-optimized)
& 494
& \textbf{39.5} \\

& \sys~(latency-optimized)
& \textbf{425}
& 29.7 \\

\midrule

6 &
\sys~(frame-rate + latency optimized) &
447 &
\textbf{44.3} \\

\bottomrule
\end{tabular}
\end{table}

\paragraph{Cost model.} Write $d(k)$ and $v(k)$ for the DiT and VAE stage latencies at SP degree $k$. WorldPlay is the pure actions~$\to$~DiT~$\to$~VAE pipeline analyzed in~\Cref{sec:analysis:app1}, and its two deployment families instantiate that analysis. Latency follows the critical path through the DiT and then the VAE. Co-locating the two on a single SP-$k$ group both shortens this path, since a higher SP degree fits within a given budget, and avoids a cross-group transfer, so it minimizes latency. Frame rate is instead set by the interval $\Delta$ between successive output chunks. Under co-location the DiT and VAE share the GPUs and run serially,
\[
\Delta_{\mathrm{co}} \;=\; d(k) + v(k),
\]
while disaggregating them onto separate groups overlaps the VAE decode of one batch with the DiT of the next,
\[
\Delta_{\mathrm{dis}} \;=\; \max\!\big(d(k_D),\, v(k_V)\big),
\]
bounded by the slower stage. Disaggregation buys this smaller interval at the cost of an inter-group transfer that lengthens the latency-binding path.

\begin{table}[t]
\centering
\caption{Placement $\rho$ for the disaggregated (frame-rate-optimized) WorldPlay deployments as the GPU budget grows from 2 to 6. The DiT and VAE occupy separate groups and are pipelined across batches; the GPUs added from 4 to 6 go to the DiT group (SP-2 to SP-4) while the VAE stays at SP-2.}
\label{tab:place_worldplay_full}
\small
\setlength{\tabcolsep}{4pt}
\renewcommand{\arraystretch}{1.1}
\begin{tabular}{@{}lP{2.6em}P{2.6em}@{\hspace{0.5em}}!{\vrule}@{\hspace{0.5em}}P{2.6em}P{2.6em}@{\hspace{0.5em}}!{\vrule}@{\hspace{0.5em}}P{2.6em}P{2.6em}@{}}
\toprule
& \multicolumn{2}{c}{\textbf{2 GPU}} & \multicolumn{2}{c}{\textbf{4 GPU}} & \multicolumn{2}{c}{\textbf{6 GPU}} \\
\cmidrule(lr){2-3} \cmidrule(lr){4-5} \cmidrule(lr){6-7}
GPU & DiT & VAE & DiT & VAE & DiT & VAE \\
\midrule
G0 & X &   & X &   & X &   \\
G1 &   & X & X &   & X &   \\
G2 & \gcell & \gcell &   & X & X &   \\
G3 & \gcell & \gcell &   & X & X &   \\
G4 & \gcell & \gcell & \gcell & \gcell &   & X \\
G5 & \gcell & \gcell & \gcell & \gcell &   & X \\
\bottomrule
\end{tabular}
\end{table}

\paragraph{What drives the tradeoff.} Two facts organize the numbers below. First, the DiT is the heavier stage at every SP degree ($d(k) > v(k)$), so under disaggregation the period $\Delta_{\mathrm{dis}} = \max(d, v) = d$ hides the VAE decode behind the DiT. Second, both stages speed up only sublinearly with the SP degree: \sys composes deployment strategies over the given operators and does not rewrite backend kernels (\Cref{sec:conclusion}), so a higher SP degree shortens a stage without ever reaching linear speedup. Latency is the serial path $d + v$, so it is set by the SP degree of each stage and is minimized by co-location, which spends the whole budget raising that degree. A useful consequence is that co-location and disaggregation at the \emph{same} per-stage SP degree reach nearly the same latency, differing only by disaggregation's small inter-group transfer; the two families otherwise trade latency against frame rate.

\paragraph{Two GPUs.} The latency-optimized deployment co-locates the DiT and VAE at SP-2, giving each stage both GPUs, and reaches the lower latency (493 ms) at 25.5 FPS. The frame-rate-optimized deployment instead runs the DiT and VAE on one GPU each (SP-1) and pipelines them, so its period $\max(d(1), v(1))$ lifts the frame rate to 31.0 FPS; its latency (625 ms) is higher, since dropping to SP-1 lengthens each stage on the serial path. Co-location thus affords the higher SP degree and wins latency, while disaggregation pipelines and wins frame rate. This is the choice analyzed in~\Cref{sec:analysis:app1}.

\paragraph{Four GPUs.} Both strategies move to a higher SP degree. The latency-optimized deployment co-locates the DiT and VAE at SP-4 and attains the lowest latency of any budget (425 ms); its co-located period $d(4)+v(4)$ also raises the frame rate to 29.7 FPS, a clear gain over the 2-GPU co-located point (25.5 FPS). The frame-rate-optimized deployment disaggregates the DiT and VAE at SP-2 each and pipelines them, reaching 39.5 FPS through the shorter period $\max(d(2), v(2))$. Its latency (494 ms) is essentially that of the 2-GPU co-located deployment (493 ms): both run each stage at SP-2, so the serial path $d(2)+v(2)$ is the same, and splitting the two stages onto separate GPU groups adds only a small transfer. This is the clean form of the tradeoff: at a matched SP degree, disaggregation costs almost nothing in latency and buys a markedly higher frame rate.

\paragraph{Six GPUs.} The 6-GPU deployment disaggregates the DiT at SP-4 from the VAE at SP-2 and pipelines them, reaching the highest frame rate (44.3 FPS) at a latency (447 ms) close to the 4-GPU co-located minimum (425 ms). Because the pipeline is DiT-bound, its period $\max(d(4), v(2)) = d(4)$ is lowered only by scaling the DiT, so the two GPUs added over the 4-GPU pipeline go to the DiT group (SP-2 $\to$ SP-4) while the VAE stays at SP-2. Latency stays near the 4-GPU minimum because both deployments run the DiT at SP-4. Returns still diminish: with sublinear SP speedup, each further doubling of the DiT group shortens $d$ by less, so frame rate grows more slowly than the GPU count.

\subsection{Video Narrator (LongLive)}
\label{sec:supp:longlive}

We use the LongLive video narrator of~\Cref{sec:exp:longlive}, in which an ASR model transcribes the user's spoken prompts while an autoregressive DiT and a streaming VAE render a continuous video stream. \Cref{tab:longlive_full} reports the full set of deployments the agent found across GPU budgets, and \Cref{tab:place_longlive} shows the node-to-GPU placements of the two multi-GPU deployments we analyze. We focus on the 4- and 8-GPU deployments and analyze how the additional GPUs affect latency and frame rate.

\begin{table}[t]
\centering
\caption{
Full set of LongLive video-narrator deployments found by the \sys agent across GPU budgets (cf.\ \Cref{tab:longlive}). The 4- and 8-GPU deployments are analyzed in~\Cref{sec:supp:longlive}.
}
\label{tab:longlive_full}

\setlength{\tabcolsep}{6pt}
\renewcommand{\arraystretch}{1.15}

\small
\begin{tabular}{c l c c}
\toprule
\textbf{\# GPUs} &
\textbf{Deployment} &
\textbf{Latency (ms)} $\boldsymbol{\downarrow}$ &
\textbf{Frame rate (FPS)} $\boldsymbol{\uparrow}$ \\
\midrule

1 &
Baseline (sequential) &
674 &
25.8 \\

\midrule

\multirow{2}{*}{2}
& \sys~(latency-optimized)
& \textbf{512}
& 36.5 \\

& \sys~(frame-rate-optimized)
& 630
& \textbf{47.8} \\

\midrule

4 &
\sys~(latency-optimized) &
\textbf{430} &
38.5 \\

\midrule

8 &
\sys~(frame-rate-optimized) &
462 &
\textbf{67.4} \\

\bottomrule
\end{tabular}
\end{table}

\paragraph{Cost model.} \Cref{sec:supp:worldplay} analyzed the actions~$\to$~DiT~$\to$~VAE pipeline of the video world model; LongLive adds an ASR stage that transcribes each new spoken prompt. The period analysis there carries over unchanged to the DiT and VAE: co-location gives $\Delta_{\mathrm{co}} = d(k) + v(k)$ and disaggregation gives $\Delta_{\mathrm{dis}} = \max\!\big(d(k_D),\, v(k_V)\big)$, with frame rate set by $\mathrm{FPS}\propto 1/\Delta$. The one addition is the ASR latency $a$ on the time-to-first-output path,
\[
L \;=\; a + d(k_D) + v(k_V).
\]
Because the ASR fires only on a new spoken prompt and runs asynchronously, it enters $L$ but not the period $\Delta$, so it does not affect the frame-rate analysis.

\begin{table}[t]
\centering
\caption{Placement $\rho$ for the two multi-GPU LongLive deployments analyzed in~\Cref{sec:supp:longlive}. The latency-optimized deployment co-locates the DiT and VAE on one SP-4 group; the frame-rate-optimized deployment disaggregates the DiT (SP-4), VAE (SP-3), and ASR onto separate groups so the VAE decode pipelines with denoising and transcription never blocks generation. In every deployment except the 8-GPU one the ASR shares G0 with the DiT.}
\label{tab:place_longlive}
\small
\setlength{\tabcolsep}{4pt}
\renewcommand{\arraystretch}{1.1}
\begin{tabular}{@{}lP{2.9em}P{2.9em}P{2.9em}@{\hspace{0.6em}}!{\vrule}@{\hspace{0.6em}}P{2.9em}P{2.9em}P{2.9em}@{}}
\toprule
& \multicolumn{3}{c}{\textbf{4 GPU (co-located)}} & \multicolumn{3}{c}{\textbf{8 GPU (disagg.)}} \\
\cmidrule(lr){2-4} \cmidrule(lr){5-7}
GPU & ASR & DiT & VAE & ASR & DiT & VAE \\
\midrule
G0 & X & X & X & X &  &  \\
G1 &  & X & X &  & X &  \\
G2 &  & X & X &  & X &  \\
G3 &  & X & X &  & X &  \\
G4 & \gcell & \gcell & \gcell &  & X &  \\
G5 & \gcell & \gcell & \gcell &  &  & X \\
G6 & \gcell & \gcell & \gcell &  &  & X \\
G7 & \gcell & \gcell & \gcell &  &  & X \\
\bottomrule
\end{tabular}
\end{table}

The streaming VAE decode scales sublinearly: $v(k)$ flattens after a low SP degree. \sys accommodates this through placement rather than by rewriting the VAE kernel, and this constraint shapes the LongLive deployments below.

\paragraph{Four GPUs.} The latency-optimized 4-GPU deployment co-locates the DiT and VAE on a single SP-4 group, with the ASR sharing G0. Because latency is the serial sum $a + d(4) + v(4)$, driving both components to the largest SP degree the budget allows yields the shortest critical path, and this deployment attains the lowest latency at any budget (430 ms, $1.6\times$ below the baseline). Its frame rate is set by $\Delta_{\mathrm{co}} = d(4) + v(4)$ and lands between the two 2-GPU deployments. It necessarily beats the co-located SP-2 point, since neither $d$ nor $v$ increases with the SP degree and so $d(4)+v(4) < d(2)+v(2)$ (38.5 vs.\ 36.5 FPS); but it does not beat the disaggregated 2-GPU pipeline, whose period is $\max(d(1), v(1))$ (47.8 FPS). Because the VAE saturates, the sum of two SP-4 latencies exceeds the maximum of two SP-1 latencies: the co-located deployment must pay the VAE term that the pipeline hides beneath the DiT. A higher SP degree alone therefore cannot match the frame rate a pipeline reaches with fewer GPUs.

\paragraph{Eight GPUs.} The frame-rate-optimized 8-GPU deployment holds the DiT at SP-4 but moves the VAE to its own SP-3 group and gives the ASR a dedicated GPU, disaggregating all three stages. Keeping the DiT at SP-4 fixes its critical-path term $d(4)$, so latency stays essentially at the 4-GPU value: $a + d(4) + v(3)$ is only marginally higher (462 vs.\ 430 ms), reflecting the VAE running at SP-3 rather than SP-4 and its output now crossing a group boundary. Frame rate, in contrast, improves sharply, because disaggregation replaces the co-located sum $d(4)+v(4)$ with the pipelined maximum $\max(d(4), v(3))$, lifting the frame rate to 67.4 FPS ($2.6\times$ the baseline) at nearly unchanged latency. The four GPUs added over the latency-optimized deployment are thus spent on disaggregation rather than on a higher SP degree: once the VAE has stopped scaling, converting the serial sum into a pipelined maximum is what buys frame rate, whereas latency keeps tracking the serial sum $a + d(k_D) + v(k_V)$ and is minimized by the SP degree. Latency and frame rate therefore respond to added GPUs through different mechanisms, so the deployment that is best for one is not best for the other.

\section{Additional Analysis of the Agent Harness}
\label{sec:supp:analysis}

This appendix provides more detailed analysis of the experiments conducted in~\Cref{sec:exp:ablation,sec:exp:robustness}. For each study, we expand on the set of deployments the agent implemented, rather than only the best ones summarized in~\Cref{sec:exp:ablation,sec:exp:robustness}. As stated previously, the results from these additional studies are based on running the \sys agent on the face-to-face conversational agent application of~\Cref{sec:exp:liveavatar} on an 8-GPU B200 node.

\subsection{Ablation}
\label{sec:supp:analysis:ablation}

\Cref{tab:ablation_full} lists the deployments produced by each ablated configuration. The Naive configuration is omitted because it terminates after a single deployment, the S2V output stream reported in~\Cref{tab:ablation}.

\begin{table}[h]
\centering
\caption{
All deployments measured by each ablated configuration on the face-to-face conversational agent. The pipeline-parallel-only variant produced by the full \sys agent does not stream, so its output is produced as a single burst and no sustained frame rate is defined for it.
}
\label{tab:ablation_full}
\small
\setlength{\tabcolsep}{6pt}
\renewcommand{\arraystretch}{1.15}
\begin{tabular}{l l c c c}
\toprule
\textbf{Configuration} & \textbf{Deployment} & \textbf{\# GPUs} & \textbf{Latency (s)} $\boldsymbol{\downarrow}$ & \textbf{Frame rate (FPS)} $\boldsymbol{\uparrow}$ \\
\midrule
\multirow{2}{*}{No IR}
 & S2V output streaming & 2 & 10.44 & 30.94 \\
 & Full streaming & 2 & 1.68 & 30.67 \\
\midrule
\multirow{2}{*}{No Loop}
 & Full streaming & 2 & 1.67 & 30.93 \\
 & Full streaming + S2V PP & 6 & 2.37 & 44.69 \\
\midrule
\multirow{3}{*}{\sys}
 & Full streaming & 2 & 1.69 & 30.92 \\
 & S2V PP & 6 & 36.29 & -- \\
 & Full streaming + S2V PP & 6 & \textbf{1.72} & \textbf{133.48} \\
\bottomrule
\end{tabular}
\end{table}

The two configurations that lack the IR pass start by implementing the same strategy, namely streaming generated video to the output buffer while leaving the TTS $\to$ S2V edge blocking. Latency is therefore at least the time required to synthesize the full response audio, so both configurations sit above 10 s. The No IR configuration escapes this bound because, through the enforced loop, the agent continues exploring after its first working variant and identifies the remaining streaming edge; without the loop, the Naive configuration does not discover this streaming opportunity.

The No Loop configuration also discovers the full streaming variant, but this happens directly during the IR phase, rather than as a result of structured iteration (as in the No IR configuration). Importantly, through the IR pass, the No Loop configuration also identifies the S2V pipeline-parallel strategy, which both the Naive and No IR configurations (which did not have an IR pass) failed to identify. The No Loop configuration proceeds directly to implement a combined variant that composes full streaming with S2V pipeline parallelism; however, the agent only prototypes pipeline parallelism within a single process, where a synchronization event on any stage stalls all stages; the agent reports a suboptimal 44.69 FPS and terminates at this prototype rather than iterating further. In contrast, the \sys configuration measures the pipeline-parallel strategy in isolation first; this variant carries no streaming and therefore has poor latency (36.29 s), but implementing this strategy in isolation allows the agent to develop a multi-process coordination mechanism that exchanges activations over NCCL, avoiding single-process blocking behavior. Then, the agent systematically composes that implementation with streaming to produce the final deployment, which improves on the No Loop result on both axes at the same GPU count.

\subsection{Sensitivity to the underlying model}
\label{sec:supp:analysis:model}

\Cref{tab:sensitivity_full} reports the two deployments produced by Codex with GPT-5.6 Sol that are relevant to the comparison in~\Cref{tab:robustness_model}.

\begin{table}[h]
\centering
\caption{
Deployments produced by Codex with GPT-5.6 Sol on the face-to-face conversational agent application.
}
\label{tab:sensitivity_full}
\small
\setlength{\tabcolsep}{6pt}
\renewcommand{\arraystretch}{1.15}
\begin{tabular}{l c c c}
\toprule
\textbf{Deployment} & \textbf{\# GPUs} & \textbf{Latency (s)} $\boldsymbol{\downarrow}$ & \textbf{Frame rate (FPS)} $\boldsymbol{\uparrow}$ \\
\midrule
S2V output streaming + S2V PP & 6 & 10.30 & 132.49 \\
Full streaming + S2V PP + VAE SP & 8 & \textbf{1.77} & 131.98 \\
\bottomrule
\end{tabular}
\end{table}

The 6-GPU deployment applies pipeline parallelism but streams only the generated video to the output buffer, leaving the TTS $\to$ S2V edge blocking, so it reaches nearly the same frame rate as the final deployment while retaining the latency of a non-streaming pipeline. Adding the TTS $\to$ S2V edge recovers the latency reduction. The final deployment also shards the VAE across three ranks, which accounts for its use of 8 GPUs: four for the denoising steps, three for the VAE, and one shared by the ASR, LLM, and TTS. This does not change either metric: frame rate under pipeline parallelism is set by the slowest stage, and a single denoising step is slower than the single-rank VAE, so accelerating the VAE cannot raise the ceiling; latency is likewise unaffected, as the reduced VAE time is offset by the coordination overhead of an additional process group.

\subsection{Consistency across runs}
\label{sec:supp:analysis:repro}

\Cref{tab:repro_full} reports the representative trajectory of each of the three independent runs summarized in~\Cref{tab:robustness_repro}.

\begin{table}[h]
\centering
\caption{
Deployments measured in each of three independent runs on the face-to-face conversational agent application. All three runs propose the same three deployments in the same order and select the composed
deployment as the best.
}
\label{tab:repro_full}
\small
\setlength{\tabcolsep}{6pt}
\renewcommand{\arraystretch}{1.15}
\begin{tabular}{l cc cc cc}
\toprule
& \multicolumn{2}{c}{\textbf{Run 1}} & \multicolumn{2}{c}{\textbf{Run 2}} & \multicolumn{2}{c}{\textbf{Run 3}} \\
\cmidrule(lr){2-3} \cmidrule(lr){4-5} \cmidrule(lr){6-7}
\textbf{Deployment} & Latency (s) & FPS & Latency (s) & FPS & Latency (s) & FPS \\
\midrule
Full streaming & 1.69 & 30.92 & 1.64 & 30.75 & 1.67 & 30.87 \\
S2V PP & 36.29 & -- & 43.10 & -- & 37.99 & -- \\
Full streaming + S2V PP & 1.72 & 133.48 & 1.63 & 133.07 & 1.78 & 125.30 \\
\bottomrule
\end{tabular}
\end{table}

The three runs also agree on how they reach the final deployment: each isolates streaming, then isolates pipeline parallelism, then composes the two. This behavior is directly induced by the \sys workflow, since the IR pass helps the agent identify both axes and the loop requires each one to be measured systematically before they are combined. The runs differ only in placement. Run 1 pairs the ASR with the VAE and the LLM with the TTS for a total of 6 GPUs, while runs 2 and 3 give each of these components its own GPU for a total of 8. Because the ASR, LLM, and TTS complete before the S2V pipeline begins to produce frames, co-locating them changes neither the critical path nor the steady-state period, and the three runs land within 0.15 s and 8.2 FPS of one another. Since the agent is given a full 8-GPU budget in all three runs, it is free to disaggregate these components, even though the final deployment does not require it.

\section{Agent Cost}
\label{sec:supp:cost}

\Cref{tab:cost} reports the token usage, monetary cost, and wall-clock time of the five B200 agent sessions that produced the deployments of~\Cref{sec:exp:liveavatar,sec:exp:generalization}. Costs are computed at current API pricing. Because the GPUs must remain allocated to the agent for the duration of the session, the corresponding compute cost is the E2E time multiplied by the GPU budget.

\begin{table}[h]
\centering
\caption{
Token usage, cost, and wall-clock time for the agent session that produced each application's deployments. Cache reads dominate token usage because the agent repeatedly revisits the reference implementation, the IR, and its own benchmark records across the session.
}
\label{tab:cost}
\small
\setlength{\tabcolsep}{5pt}
\renewcommand{\arraystretch}{1.15}
\begin{tabular}{l r r r r r r r}
\toprule
\textbf{Application} &
\textbf{Input} &
\textbf{5\,min cache} &
\textbf{1\,hr cache} &
\textbf{Cache read} &
\textbf{Output} &
\textbf{Cost} &
\textbf{E2E (hr)} \\
\midrule
LiveAvatar & 44,500 & 155,511 & 1,759,657 & 246,157,708 & 1,097,179 & \$166.98 & 4.99 \\
Qwen3-Omni & 31,835 & 88,267 & 1,074,911 & 248,093,750 & 701,940 & \$151.92 & 4.92 \\
Krea-Realtime + SAM 3 & 22,646 & 51,096 & 747,557 & 110,365,086 & 485,026 & \$74.53 & 2.64 \\
WorldPlay & 28,722 & 0 & 1,003,877 & 222,005,520 & 631,620 & \$136.98 & 3.51 \\
LongLive & 46,251 & 0 & 972,942 & 204,161,864 & 588,008 & \$126.74 & 5.43 \\
\bottomrule
\end{tabular}
\end{table}

At \$75 to \$167 per application, an agent session is inexpensive relative to the expert engineering time it replaces, and the comparison extends past the direct cost. Writing efficient serving code for a new real-time pipeline otherwise depends on the availability of engineers with the relevant systems expertise, which does not scale with the number of applications; instead, agent-driven deployment scales with compute that can be provisioned per application.

These costs reflect a search in which the agent implements and measures nearly every candidate it queues. It discards a queued candidate only in limited circumstances when a candidate is deemed unviable, and the agent always provides a recorded justification. For example, we observe a small number of instances where the agent attempts a deployment that uses an invalid SP degree that does not evenly divide the number of latent frames processed by a video model; in such cases, the agent identifies the error and discards the candidate, then replaces it with a valid candidate that uses a viable SP degree.